%% file: iclr2025_conference.tex
\documentclass{article} % For LaTeX2e
\usepackage{iclr2025_conference,times}

\input{math_commands.tex}

\usepackage{xcolor}
\usepackage{hyperref}
\usepackage{tikz}
\usepackage{url}
\usepackage{amsmath}
\usepackage{amssymb}
\usepackage{graphicx}
\usepackage{array}
\usepackage{algorithm}
\usepackage{algpseudocode}
\usepackage{makecell}
\usepackage{multirow}
\usepackage{wrapfig}
\usepackage{dsfont}
\usepackage{pifont}
\usepackage{authblk}
\usepackage{amsthm}
\usepackage{booktabs}
\usepackage{wrapfig}

\usepackage{natbib}
\usepackage{tcolorbox}
\usepackage{verbatim}
\usepackage{listings}
\usepackage{subfig}

\lstdefinestyle{stateSnapshot}{
  basicstyle=\ttfamily\fontsize{7}{8}\selectfont,
  breaklines=true,
  frame=single,
  backgroundcolor=\color{blue!3},
  keywordstyle=\color{blue},
  stringstyle=\color{red!60!black},
  commentstyle=\color{green!50!black},
  showstringspaces=false,
  tabsize=2,
  numbers=none
}

\newcommand{\cmark}{\ding{51}}  % ✓
\newcommand{\xmark}{\ding{55}} 
\usepackage{enumitem}
\DeclareMathOperator*{\argmaxA}{\arg\max}

\title{ANDRE: An Attention-based Neuro-symbolic Differentiable Rule Extractor for Inductive Logic Programming}

\iclrfinalcopy
\author[1]{Iman Sharifi}
\author[1]{Peng Wei}
\author[2]{Saber Fallah}
\affil[1]{Dept. of Mechanical and Aerospace Engineering, George Washington University, USA}
\affil[2]{Dept. of Mechanical Engineering Sciences, University of Surrey, UK }
\affil[ ]{\texttt{\{i.sharifi,pwei\}@gwu.edu, s.fallah@surrey.ac.uk}}

\newtheorem{theorem}{Theorem}
\newtheorem{lemma}{Lemma}

\begin{document}

\maketitle

\begin{abstract}
Inductive Logic Programming (ILP) aims to learn interpretable first-order rules from data, but existing symbolic and neuro-symbolic approaches struggle to scale to noisy and probabilistic settings. Classical ILP relies on discrete combinatorial rule search and is brittle under uncertainty, while differentiable ILP methods typically depend on predefined rule templates or inaccurate fuzzy operators that suffer from vanishing gradients or poor approximation of logical structure when reasoning over probabilistic predicate valuations. This paper proposes an \textit{Attention-based Neuro-symbolic Differentiable Rule Extractor (ANDRE)}, a novel ILP framework that learns first-order logic programs by optimizing over a continuous rule space with attention-based logical operators. ANDRE replaces both rule templates and logical operators with fully differentiable, attention-driven conjunction and disjunction operators that approximate logical min–max semantics, enabling accurate, stable, and interpretable reasoning over probabilistic data. By softly selecting, negating, or excluding predicates within each rule, ANDRE supports flexible rule induction while preserving symbolic structure. Extensive experiments on classical ILP benchmarks, large-scale knowledge bases, and synthetic datasets with probabilistic predicates and noisy supervision demonstrate that ANDRE achieves competitive or superior predictive performance while reliably recovering correct symbolic rules under uncertainty. In particular, ANDRE remains robust to moderate label noise, substantially outperforming existing differentiable ILP methods in both rule extraction quality and stability.
\end{abstract}

\section{Introduction}
Symbolic artificial intelligence (AI) has long provided a foundation for interpretable reasoning, enabling systems to derive logical conclusions from structured knowledge~\citep{Colelough2025NeuroSymbolicAI,garcez2023neurosymbolic}. A prominent approach within this paradigm is Inductive Logic Programming (ILP), which combines machine learning with first-order logic to infer human-readable rules from examples~\citep{muggleton2012ilp}. Given a set of examples and background knowledge, ILP aims to discover logical clauses that explain observed facts, making it an essential tool in knowledge-base systems, explainable AI~\citep{socher2013reasoning,yang2017differentiable}, and even safety-critical domains such as autonomous driving~\citep{sharifi2023symbolic} and aviation~\citep{acharya2025integrating}. 

Classical ILP methods—such as Metagol~\citep{cropper2015logical,muggleton2015meta,inoue2009discovering} and Popper~\citep{cropper2021learning}—tackle the ILP problem using a variety of heuristic, search-based strategies. While these approaches produce fully interpretable results, a major limitation lies in their sensitivity to noise in the data labels~\citep{evans2018learning}. In addition, the rule discovery process typically requires considerable user intervention to define appropriate language bias—e.g., mode declarations and constraints—making the overall pipeline cumbersome and labor-intensive~\citep{cropper2022inductive}. Furthermore, the lack of seamless integration with neural networks restricts the scalability and adaptability of classical ILP methods in modern, large-scale machine learning applications.

To address these challenges, neuro-symbolic ILP methods have emerged, aiming to integrate ILP with neural networks to enhance scalability and robustness~\citep{cropper2022inductive}. By leveraging differentiable optimization, these approaches formulate rule induction as a gradient-based learning problem, which facilitates more effective handling of noisy data and redundant predicates. Broadly, these methods fall into two categories: those that assume prior knowledge of the rules, and those that operate entirely without it. In the first category, rule induction is guided by candidate rules typically generated through general rule templates~\citep{evans2018learning} or beam search techniques~\citep{shindo2021differentiable,shindo2023alpha}. Although these approaches can efficiently identify plausible rules, they often depend on domain-specific knowledge or involve ad hoc search procedures. Furthermore, due to their reliance on prior rule generation using templates and the manual conversion of background knowledge into readable valuations, these methods are not fully differentiable, limiting their compatibility with end-to-end neural network architectures~\citep{evans2018learning}.

In contrast, the second category assumes no prior knowledge of candidate rules~\citep{yang2017differentiable,gao2022learning,gao2024differentiable}. Instead, these methods explore rules by assigning trainable weights to vector or matrix embeddings that represent potential predicate valuations. While these approaches are  differentiable, partially robust to noise, and free from reliance on predefined rule templates, they mostly depend on soft, inaccurate approximations of logical conjunctions and disjunctions, using product-based $t$-norms and $s$-norms~\citep{yang2017differentiable,de2008probabilistic} and regularized geometric mean~\citep{wang-pan-2022-deep}, or trainable, less-interpretable parameterized function approximators~\citep{serafini2016logic,sen2022neuro}. Although product-based operators are accurate for Boolean inputs~\citep{sen2022neuro,payani2019inductive}, their performance deteriorates with probabilistic data, where the product tends to rapidly decay toward zero~\citep{hajek2001metamathematics}, leading to vanishing gradient issues~\citep{shakarian2023neuro}. Parameterized approximators, on the other hand, often sacrifice interpretability by representing logical operations as complex nonlinear functions~\citep{richardson2006markov,dong2018neural}. To date, attention-based neuro-symbolic ILP approaches, such as NeuralLP~\citep{yang2017differentiable} and DILR~\citep{wang-pan-2022-deep}, have emerged to overcome the limitations of the existing approaches, integrating attention mechanisms~\citep{vaswani2017attention} with logical operators. However, despite their effectiveness in general reasoning and question answering, their internal architectures still rely on classical logical operators, thereby inheriting the associated limitations. Given the limitations of existing neuro-symbolic ILP methods, there is a clear need for an alternative approach---one that extracts rules without predefined templates, avoids inefficient logical operators, and maintains both interpretability and structural flexibility.

This paper proposes a novel neuro-symbolic ILP method, called \textbf{\textit{Attention-based Neuro-symbolic Differentiable Rule Extractor (ANDRE)}}, which learns first-order logic programs from positive and negative examples using a fully differentiable and interpretable architecture. Operating without any assumptions about candidate rules or search-based heuristics, ANDRE reframes rule induction as learning over a continuous rule space in which each body predicate is softly selected, negated, or excluded using a novel fuzzy activation function. Furthermore, ANDRE employs softmin/softmax attention-based operators instead of classical logical operators, to effectively represent conjunction/disjunction operators. To the best of our knowledge, this is the first work to explicitly use softmin and softmax attention as a differentiable approximation of conjunction and disjunction in ILP. ANDRE's novel attention-based operators are both accurate and fully differentiable. These modules are well-suited not only for Boolean data but also for continuous probabilistic data and mitigate vanishing gradient issues that typically hinder differentiable rule learners. Moreover, ANDRE is both scalable to large knowledge bases and strongly robust to noise in the data labels. %To further enhance rule extraction, ANDRE incorporates a brand-new curriculum learning framework that discovers subrules in a progressive, accuracy-guided manner. The approach is inherently interpretable, as the learned weights highlight the structure and significance of each predicate, ensuring compliance with subrule formatting constraints. 
% Our main contributions are as follows:
% \begin{itemize}
% \item We propose ANDRE, a novel logic-based network that explores a well-crafted rule space and introduces adaptive, attention-based conjunction and disjunction operators—enabling robustness, flexibility, and noise tolerance while improving accuracy on Boolean and probabilistic data.
% % \item ANDRE features an efficient learning framework that can not only learn subrules progressively but also do one-shot training, where the learned weights are fully interpretable. 
% \item We conduct comprehensive experiments comparing ANDRE with state-of-the-art methods, demonstrating its competitive performance over classical ILP datasets as well as superior performance over knowledge bases and complex synthetic probabilistic datasets, in terms of rule extraction accuracy and generalization—even under noisy conditions.
% \end{itemize}
The main contributions of this paper are:
\vspace{-0.5em}
\begin{itemize}[leftmargin=1.5em,itemsep=0.1em]
\item First, we introduce a new continuous rule space formulation for ILP, in which the discrete inclusion, negation, and exclusion of predicates are relaxed into differentiable probability distributions, enabling gradient-based learning of symbolic programs without rule templates.

\item Second, we propose attention-based logical operators that provide a differentiable approximation to symbolic min–max conjunction and disjunction, allowing logical reasoning to be performed reliably over probabilistic predicate valuations rather than crisp Boolean facts.

\item Third, we derive a set of syntactic regularization losses, with a loss hyperparameter scheduling technique, that enforce variable inclusion, connectivity, and digitization, ensuring that the learned programs correspond to valid and extractable first-order logic rules.

\item Finally, through extensive experiments on classical ILP benchmarks, large-scale knowledge bases, and synthetic datasets with noisy supervision, we demonstrate that ANDRE achieves competitive or superior predictive performance while significantly improving rule extraction quality and robustness under uncertainty compared to existing differentiable ILP methods.
\end{itemize}
The remainder of the paper is organized as follows. Section~\ref{sec:preliminaries} reviews the foundational concepts and necessary background. Section~\ref{sec:methodology} introduces ANDRE's methodology in detail. Section~\ref{sec:experiments} presents experimental evaluations. Finally, Section~\ref{sec:conclusion} draws conclusions.

\section{Preliminaries}
\label{sec:preliminaries}
\textbf{Inductive Logic Programming (ILP).}
As a subfield of machine learning, ILP focuses on learning logical rules from structured data. Given a background knowledge base \( \mathbb{B} \), a set of positive examples \( E^+ \), and a set of negative examples \( E^- \), the objective of ILP is to learn a hypothesis \( H \) such that:
\begin{equation}
\forall e \in E^+ : \mathbb{B} \cup H \models e, \quad
\forall e \in E^- : \mathbb{B} \cup H \not\models e,
\label{eq:ilp}
\end{equation}
where \( \models \) denotes logical entailment. The hypothesis \( H \) comprises a set of first-order logic rules that best explain the observed examples while remaining consistent with the background knowledge. A typical ILP rule is a same-head disjunction composed of several subrules, defined as:
\begin{equation}
\texttt{h}(X_1^h, X_2^h, ..., X_{n_h}^h) \leftarrow \bigvee_{i=1}^n \texttt{h}_i(X_1^h, X_2^h, ..., X_{n_h}^h),
\label{rule_form}
\end{equation}
where \( \texttt{h} \) is the target head predicate, and the arguments \( X_k^h \in \mathbb{R} \) are the head variables shared across all subrules. Inheriting all head variables from \( \texttt{h} \), each subrule \( \texttt{h}_i \) is defined as a conjunction of specific body predicates:
\begin{equation}
\texttt{h}_i(X_1^h, X_2^h, ..., X_{n_h}^h) \leftarrow \bigwedge_{j=1}^m \texttt{B}_{ij}(X_1^{b_j}, X_2^{b_j}, ..., X_{n_{b_j}}^{b_j}),
\label{subrule_form}
\end{equation}
where \( \texttt{B}_{ij} \) represents a potential body predicate in \( \texttt{h}_i \), and \( X_k^{b_j} \in \mathbb{R} \) are the variables associated with the \( j \)-th predicate. The arity \( n_{b_j} \) can vary depending on the predicate. The ILP objective is to identify the optimal assignment of \( \texttt{B}_{ij} \) in every subrule \( \texttt{h}_i \), ultimately constructing the complete rule structure in Eq.~\ref{rule_form}. Due to the combinatorial nature of subrules, which involve various configurations of body predicates, optimizing them using neural networks—typically designed for fixed-size inputs—poses significant challenges.

\textbf{Neuro-symbolic Inductive Logic Programming.}
Neuro-symbolic ILP aims to combine the symbolic reasoning capabilities of traditional ILP systems with the robustness and adaptability of neural networks. These methods reformulate the ILP task as a continuous, differentiable process, casting rule induction as an optimization problem in which logical rules are learned by minimizing a loss function over a dataset of \( N \) labeled examples \( E^{N \times (m+1)} \). Each example \( e \in [0,1]^{(m+1)} \) includes the valuations of \( m \) potential body predicates, along with a corresponding binary output \( h \in \{0,1\} \). The goal is to map the inputs to outputs while discovering explicit logical connections between the body predicates and the target head predicate. This formulation enables training via gradient descent, allowing the model to generalize from ambiguous patterns and tolerate noise or imperfections in the data.

From a technical perspective, neuro-symbolic ILP typically maps positive and negative examples to binary labels, transforming the problem into a binary classification task that can be addressed using standard loss functions. The total loss function commonly consists of two main components: (1) \textit{semantic loss} and (2) \textit{syntactic loss}. Semantic loss—typically formulated as binary cross-entropy—guides the model in identifying high-level relationships between body predicates and the head predicate. It evaluates which subpredicates contribute most to reducing prediction error. Syntactic loss, in contrast, enforces structural correctness by ensuring that the learned rules conform to valid logical formats and general ILP constraints.

\section{Attention-based Neuro-symbolic Rule Extractor Architecture}
\label{sec:methodology}
ANDRE learns logical rules from structured examples by introducing a rule space that enables smooth and interpretable rule induction. It also employs attention-based conjunction and disjunction mechanisms that approximate symbolic reasoning operations. Building on these components, a learning strategy is developed to discover subrules sequentially, followed by a detailed procedure for symbolic rule extraction after training.

\subsection{Rule Space Construction}
\label{subsec:rule-space}
Since the rule and subrule structures vary dynamically in ILP, it is highly challenging to develop a universal, systematic neural network-based approach for identifying subrules for a given head predicate. This challenge is more akin to structure induction or representation learning in natural language processing, where models must infer latent semantic and syntactic relationships from variably structured prompts. To address this issue, a zero-padding technique is typically employed to enforce a fixed-size input compatible with neural networks. Adopting a similar technique in ILP mitigates the difficulty caused by the dynamic transformation of logical subrules, enabling a fixed structure suitable for neural modeling. Just as adding zero to neural inputs has a neutral effect, adding one to a logical conjunction also has no effect, due to the logical \textit{identity property} ($1 \wedge b = b$). Based on this property, we represent the exclusion of $\texttt{B}_{ij}$ from subrule $\texttt{h}_i$ in Eq.~\ref{subrule_form} using the symbolic subpredicate $\texttt{1}$. Thus, each $\texttt{B}_{ij}$ can take only one of three possible symbolic subpredicates: $\texttt{b}_j$, $\neg \texttt{b}_j$, or $\texttt{1}$. Here, \( \texttt{b}_j \) denotes the positive form of the predicate, \( \neg \texttt{b}_j \) its negation, and \( \texttt{1} \) the absence of the predicate.
\begin{wrapfigure}{r}{0.6\textwidth}
    \centering
    \includegraphics[width=\linewidth, clip, trim=0.7cm 0.28cm 1.17cm 0.33cm]{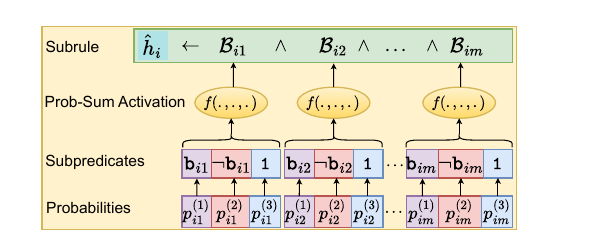}
    \vspace{0.1mm}
    \caption{\textbf{Graphical representation} of a logical subrule structure within the context of the rule space.}
    \label{fig:rule_structure}
    \vspace{1mm}
\end{wrapfigure}

Following this strategy, we reformulate the ILP problem by converting Eqs.~\ref{rule_form} and~\ref{subrule_form} into a matrix representation. Let the matrix $\texttt{B}^{n \times m}$ represents $m$ body predicates across $n$ subrules. As discussed, each array $\texttt{B}_{ij}$ is:
\begin{equation}
\texttt{B}_{ij} = \left\{ \texttt{b} \mid \texttt{b} \in  \{\texttt{b}_j, \neg \texttt{b}_j, \texttt{1}\} \right\}.
\label{subpredicates}
\end{equation}
Let $\texttt{S}_j = \{\texttt{b}_j, \neg \texttt{b}_j, \texttt{1}\}$; then $\texttt{B}_{ij} \in \texttt{B}$ refers to the target symbolic predicate to be determined. Let $b_j$, $1 - b_j$, and $1$ be the valuations of $\texttt{b}_j$, $\neg \texttt{b}_j$, and $\texttt{1}$, respectively. There are $3^m \times n$ possible combinations of subpredicates for $\texttt{B}$, making the rule space incredibly large. The challenge is to navigate this large discrete space using gradient-based optimization to recover the desired $\texttt{B}$ given a set of examples $E$ containing all $b_j$ truth values and corresponding binary outputs.

To overcome this, we convert the discrete space into a continuous one and apply gradient-based methods to learn optimal subpredicates. Specifically, we define an adaptive probability distribution over the subpredicates in $\texttt{S}_j$ to determine the most likely option. A trainable weight vector $\mathcal{W}_{ij} \in \mathbb{R}^3$ is assigned to each $\texttt{S}_j$, which is then passed through a softmax function to produce a probability distribution $P_{ij} \in [0,1]^3$, where $P_{ij} = \{ p_{ij}^{(1)}, p_{ij}^{(2)}, p_{ij}^{(3)} \}$ and \( \sum_{k=1}^3 p_{ij}^{(k)} = 1 \).

Let $\mathcal{S}_j = \{ p_{ij}^{(1)} b_j, \; p_{ij}^{(2)} (1-b_j),\; p_{ij}^{(3)} \}$ be the set of weighted valuations for the subpredicates in $\texttt{S}_j$ (illustrated in Figure~\ref{fig:rule_structure}). The goal is to amplify the most significant subpredicate while suppressing the others. While a hard max operator would select the largest value in $\mathcal{S}_j$, it introduces non-differentiability. Instead, we use a differentiable approximation based on the probabilistic sum (prob-sum), which emphasizes the dominant input among probabilistic values (see Appendix~\ref{ap:prob-sum-justification} for further details):
\begin{equation}
f(x, y, z) = x + y + z - xy - xz - yz + xyz,
\label{prob_sum}
\end{equation}
where $f: [0,1]^3 \rightarrow [0,1]$ is the prob-sum activation function applied to $x, y, z \in [0,1]$. Using this function, the valuation of $\texttt{B}_{ij}$ is computed as:
\begin{equation}
\mathcal{B}_{ij} = f \bigl(p_{ij}^{(1)} b_j,\, p_{ij}^{(2)} (1-b_j),\, p_{ij}^{(3)}\bigl),
\end{equation}
where $\mathcal{B}_{ij}$ denotes the soft valuation of $\texttt{B}_{ij}$. This formulation quantifies the combined contribution of the three symbolic subpredicates and emphasizes the most dominant one.

Figure~\ref{fig:rule_structure} visually demonstrates how the probabilities are assigned and then aggregated via Eq.~\ref{prob_sum} to compute $\mathcal{B}_{ij}$. This step is essential for constructing the continuous rule space. The overall objective is to efficiently explore this parameterized space and discover a set of subrules that accurately explain the positive and negative examples provided in the training set.
 
\begin{figure}[!t]
    \centering
    \includegraphics[width=\linewidth, clip, trim=0.7cm 0.35cm 0.7cm 0.22cm]{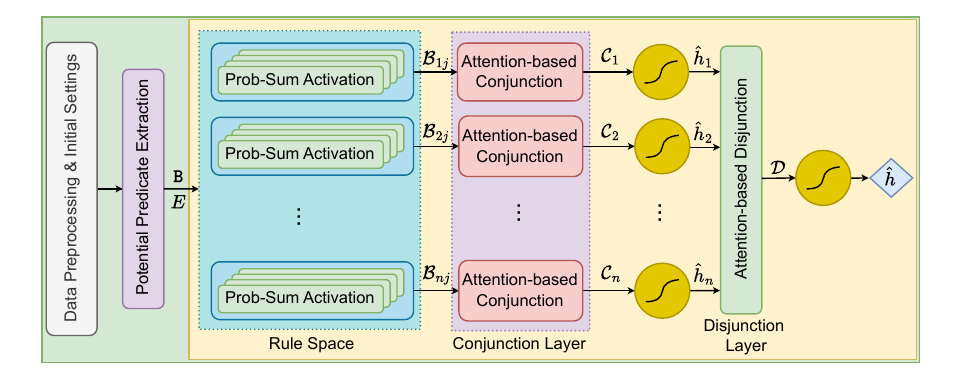}
    \vspace{-2mm}
    \caption{\textbf{Overview of ANDRE's architecture.} ANDRE includes an innovative rule space with an \textit{Attention-based Conjunction-Disjunction} Network.}
    \label{fig:ANDRE}
\end{figure}

\subsection{Logical Network of ANDRE}
Having constructed the continuous rule space with trainable probabilities, we now describe ANDRE's logical network. Following the computation of all body predicate valuations, ILP systems typically apply a logical conjunction-disjunction mechanism to infer the value of the head predicate. Recently, \cite{hu2025minimalist} demonstrated that a single softmax attention layer can approximate $k$-Boolean functions (e.g., AND/OR). Inspired by this work, we adopt and customize a new strategy to approximate hard minimum and maximum operators using soft attention.

\textbf{Conjunction Layer.}
As shown in Figure~\ref{fig:ANDRE}, after constructing the rule space and computing $\mathcal{B}_{ij}$ for subrule $\texttt{h}_i$, the valuation of the subrule head is determined by aggregating the values of all $\mathcal{B}_{ij}$ through a conjunction operation. The attention-based conjunction is calculated as follows:
\begin{equation}
\mathcal{C}_i = \sum_{j=1}^m \frac{e^{s_{\text{min}}(\mathcal{B}_{ij})}}{\sum_{k=1}^m e^{s_{\text{min}}(\mathcal{B}_{ik})}} \mathcal{B}_{ij},
\label{conjunction}
\end{equation}
where \( \mathcal{C}_i \in [0, 1] \) represents the soft conjunction of the body predicates in subrule $\texttt{h}_i$. This mechanism uses adaptive attention scores \( s_{\text{min}}: [0,1] \rightarrow (-\infty, 0] \), which assign higher attention to smaller values:
\begin{equation}
s_{\text{min}}(\mathcal{B}_{ij}) = -\beta \mathcal{B}_{ij},
\label{min-score}
\end{equation}
where \( \beta \in (0, \infty) \) is a hyperparameter controlling the sharpness of the attention distribution. Higher values of \( \beta \) improve the approximation to a hard minimum but may introduce instability. Eq.~\ref{conjunction} yields a weighted average of the $\mathcal{B}_{ij}$ terms, where the weights are given by a softmax over the negative attention scores. Figure~\ref{fig:att-conj-disjunction} illustrates how each $\mathcal{C}_i$ is computed using the attention-based conjunction. 

Since the output of each subrule is expected to be binary, a sigmoid activation function is applied to convert \( \mathcal{C}_i \) into a value close to $0$ or $1$: \(\hat{h}_i = \text{sigmoid}\left( \lambda(\mathcal{C}_i - \gamma) \right)\), where \( \lambda \in (0, \infty) \) controls the steepness of the sigmoid function and \( \gamma \in (0, 1) \) determines its midpoint. This process is repeated for each subrule \( \texttt{h}_i \), yielding predicted outputs \( \hat{h}_i \in [0, 1] \), which are then used in the disjunction step.

\textbf{Disjunction Layer.}
Once all subrule outputs \( \hat{h}_i \) are computed, they must be combined to infer the overall head predicate value, as shown in Figure~\ref{fig:ANDRE}. To do this, we apply a similar attention-based strategy that approximates the maximum among the subrule outputs. The disjunction is computed as:
\begin{equation}
\mathcal{D} = \sum_{i=1}^n \frac{e^{s_{\text{max}}(\hat{h}_i)}}{\sum_{k=1}^n e^{s_{\text{max}}(\hat{h}_k)}} \hat{h}_i,
\label{disjunction}
\end{equation}
where \( \mathcal{D} \in [0, 1] \) represents the soft disjunction of all subrule outputs. The attention scores \( s_{\text{max}}: [0,1] \rightarrow [0, \infty) \) are defined in opposition to those in Eq.~\ref{min-score}:
\(s_{\text{max}}(\hat{h}_i) = -s_{\text{min}}(\hat{h}_i)\).
This formulation emphasizes higher-valued subrules, allowing the model to approximate a logical OR operation. Figure~\ref{fig:att-conj-disjunction} indicates how $\mathcal{D}$ is computed using an attention-based disjunction. As with conjunction, a sigmoid activation is applied to map the disjunction output to a binary value:
\( p(\hat{h} \mid e, \mathcal{W}) = \text{sigmoid}\left( \lambda(\mathcal{D} - \gamma) \right) \),
where \( p(\hat{h} \mid e, \mathcal{W}) \in [0,1] \) is the final predicted probability of the head predicate for input example $e$ with ANDRE’s weight configuration \( \mathcal{W} \in \mathbb{R}^{n \times m \times 3} \), which includes all subpredicate weights \( \mathcal{W}_{ij} \).

\begin{figure*}
    \centering    \includegraphics[width=\linewidth]{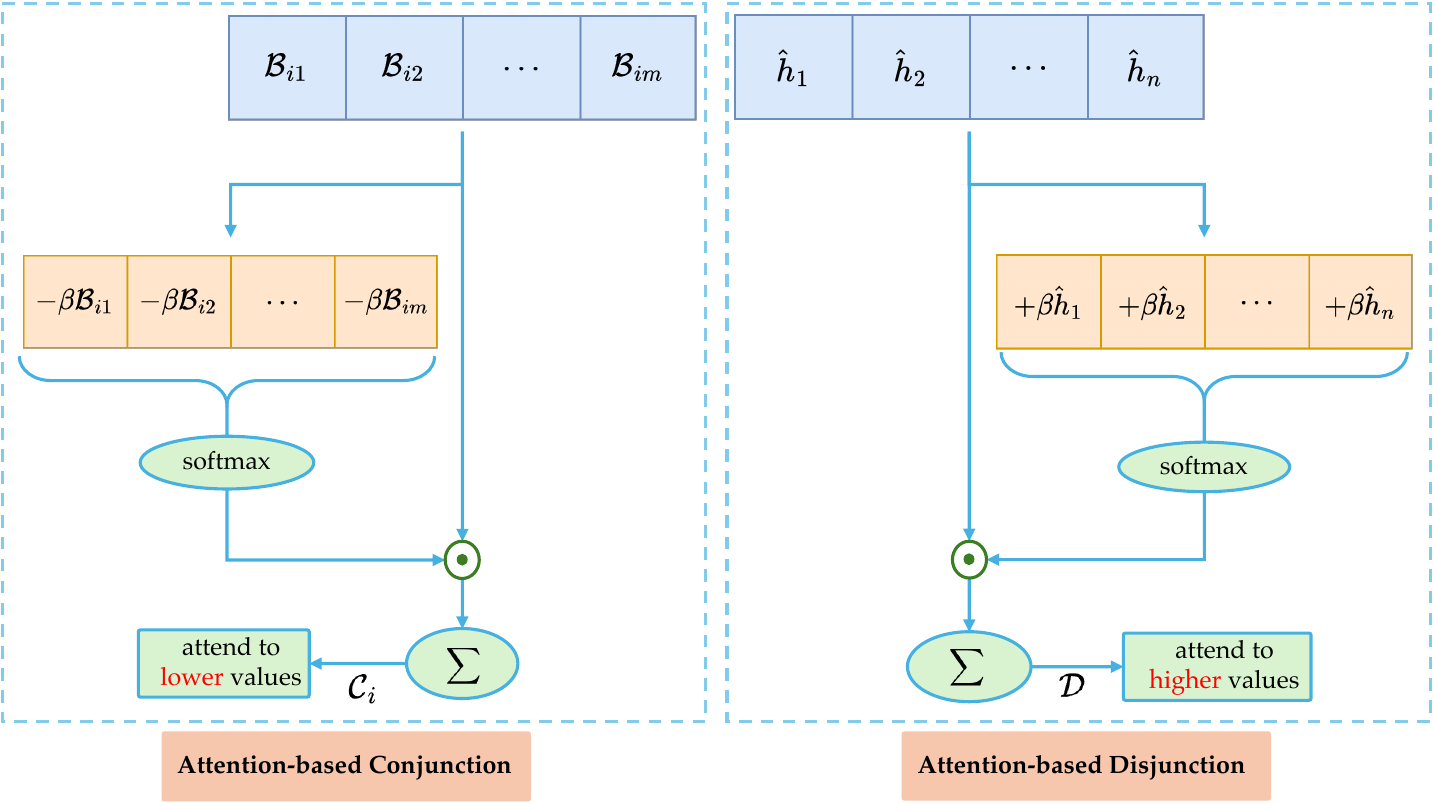}
    \vspace{-1mm}
    \caption{Graphical representations of attention-based conjunction and disjunction operators. Softmax with $-\beta$ represents softmin attention (conjunction), while softmax with $+\beta$ stands for softmax attention (disjunction). Symbol $\bigodot$ indicates a dot-product operator.}
    \label{fig:att-conj-disjunction}
\end{figure*}

Figure~\ref{fig:ANDRE} illustrates the complete ANDRE architecture, including the construction of the rule space and the attention-based conjunction and disjunction layers. This design enables symbolic rule extraction through a fully differentiable neural structure that learns logical dependencies directly from binary-labeled examples. Moreover, Appendix~\ref{ap:attention-product} mathematically proves that the proposed attention-based operators approximate hard minimum and maximum operators more accurately compared to product-based $t$-norms and $s$-norms on probabilistic data.

\subsection{Optimization}
Given the valuations of all potential body predicates and corresponding Boolean outputs, ANDRE is designed not only to identify the most relevant subpredicates but also to ensure that the resulting rules adhere to valid logical formats. To achieve this, the learning process is guided by two complementary objectives: \textit{semantic learning} and \textit{syntactic learning}.

\textbf{Semantic Learning.} This learning paradigm aims to learn a set of optimal weights $\mathcal{W}$ that enable the network to accurately predict the Boolean outputs while highlighting the most confident subpredicates. It involves optimizing three general loss functions:

\textit{Binary Cross-Entropy Loss:} This loss function evaluates the confidence of each prediction. By optimizing this loss, ANDRE learns to predict Boolean outcomes for each sample, ensuring consistency with the training examples. The loss function tailored for ANDRE is defined as:
\begin{equation}
    \mathcal{L}_{\text{BCE}} = \sum_{(e, h) \in \mathcal{D}} - h \cdot \log p(\hat{h} \,|\, e, \mathcal{W})  - (1 - h) \cdot \log \left(1 - p(\hat{h} \,|\, e, \mathcal{W})\right),
\label{eq:bce-loss}
\end{equation}
where \( p(\hat{h} \, | \, e, \mathcal{W}) \) denotes the predicted probability of the head predicate given input example \( e \) and weight configuration \( \mathcal{W} \), and the Boolean output \( h \in \{0, 1\} \) is the corresponding ground-truth label.

\textit{Entropy Loss:} To encourage exploration and prevent premature convergence to overly confident or arbitrary predicate selections, we introduce an entropy-based regularization term. For each trainable probability distribution $P_{ij}$ associated with $\texttt{S}_j$, the total entropy loss across all probabilities in all subrules is defined as:
\begin{equation}
\mathcal{L}_{\text{E}} = -\sum_{i=1}^{n} \sum_{j=1}^{m} \sum_{k=1}^{3} p_{ij}^{(k)} \cdot\log \left( p_{ij}^{(k)} + \epsilon \right),
\label{eq:entropy-loss}
\end{equation}
where $\epsilon$ is a small constant (e.g., $10^{-6}$) for numerical stability. By minimizing this loss, ANDRE encourages one of the probabilities in $P_{ij}$ to converge to $1$, while the remaining two converge to $0$. 

\textit{Similarity Loss:} To ensure subrule diversity, we minimize the cosine similarity of each subrule weights with the rest of subrule weights, defined as:
\begin{equation}
\mathcal{L}_{\text{S}} =
\frac{2}{n(n+1)}
\sum_{i=1}^{n}
\sum_{j=i+1}^{n}
% \frac{1}{d}
% \sum_{r=1}^{d}
\frac{
\left\langle W_i \, ,\, W_j \right\rangle
}{
\left\| W_i \right\| \cdot
\left\| W_j \right\|
},
\end{equation}
where $W_i$ and $W_j$ are vectorized weights of subrule $i$ and $j$, respectively.

In general, by minimizing \( \mathcal{L}_{\text{BCE}} \), \( \mathcal{L}_{\text{E}} \), and \( \mathcal{L}_{\text{S}} \), ANDRE iteratively updates its parameters to emphasize the most informative subpredicates and suppress irrelevant or redundant ones. This semantic optimization enables the model to learn interpretable and high-performing logical rules.

\textbf{Syntactic Learning.} In addition to achieving high predictive accuracy, it is essential that the extracted rules conform to valid logical structures. Two common formats are considered to ensure syntactic validity: (1) the \textit{head variable inclusion} format, and (2) the \textit{auxiliary variable connectivity} format (see Appendix~\ref{app:syntactic-formats} for more details).
To enforce these structural constraints during training, we introduce two syntactic loss functions: the \textit{Range-Restricted Loss} \( \mathcal{L}_{\text{R}} \), corresponding to head variable inclusion, and the \textit{Connected Loss} \( \mathcal{L}_{\text{C}} \), corresponding to auxiliary variable connectivity.

To compute \( \mathcal{L}_{\text{R}} \) and \( \mathcal{L}_{\text{C}} \), we first quantify how frequently each variable \( X_k \) appears in the body of a subrule \( \hat{h}_i \), given the current weights \( \mathcal{W} \). This frequency is denoted by \( M_k^i \), representing the expected count of variable \( X_k \) in the body predicates of subrule \( \hat{h}_i \).
The computation begins by identifying which body predicates contain \( X_k \) as an argument. For this, we define a Boolean indicator \( V_k^{ij} = \mathds{1}(x = X_k, \texttt{b} = \texttt{b}_j) \), which evaluates to 1 if variable \( X_k \) appears in predicate \( B_{ij} \), and 0 otherwise. This information is domain-specific and assumed to be fixed for all subrules.
We then estimate the probability that each predicate \( \texttt{B}_{ij} \) is active in subrule \( \hat{h}_i \). Since \( p_{ij}^{(3)} \) represents the probability of selecting the identity subpredicate \( \texttt{1} \), the effective inclusion probability of predicate \( \texttt{B}_{ij} \) is \( 1 - p_{ij}^{(3)} \). The expected usage count of variable \( X_k \) in \( \hat{h}_i \) is then computed as:
\begin{equation}
    M_k^i = \sum_{j=1}^m (1 - p_{ij}^{(3)}) \cdot \mathds{1}(x = X_k, \texttt{b} = \texttt{b}_{j}).
\end{equation}

\textit{Range-Restricted Loss:} To enforce the head variable inclusion constraint, we penalize subrules where a head variable is missing from the body. Let \( K_h \) denote the set of head variables. The loss is formulated as:
% \begin{equation}
% \mathcal{L}_{\text{R}} =\frac{1}{|K_h|} \sum_{k \in K_h} \text{sigmoid} \left( -\lambda_{\text{R}} (M_k^i - \gamma_r) \right)
% \end{equation}
\begin{equation}
\mathcal{L}_{\text{R}} =
% \frac{1}{|K_h| \, n}
\sum_{i=1}^{n} \sum_{k \in K_h}
\begin{cases}
\left(M_k^i - 1\right)^2, & \text{if } M_k^i < 1, \\[6pt]
\eta \left(M_k^i - 1\right), & \text{if } M_k^i \ge 1,
\end{cases}
\label{rr-loss}
\end{equation}
which ensures the inclusion of head variables at least once in the body predicates. The loss increases when \( M_k^i < 1 \), encouraging inclusion of all head variables. Parameter $\eta \in [0,\infty)$ adjusts the magnitude of penalty when \( M_k^i \ge 1 \). Higher $\eta$ values encourage the head variable to appear in body predicates only once.
% where \( \lambda_{\text{R}} \) controls the sharpness of the sigmoid function and \( \gamma_r \in (0,1) \) defines the decision boundary. 

\textit{Connected Loss:} To enforce auxiliary variable connectivity, we penalize subrules where an auxiliary variable appears only once and encourage subrules to have two auxiliary variables. Let \( K_a \) denote the set of auxiliary variables. The corresponding loss is defined as:
% \begin{equation}
%     \mathcal{L}_{\text{C}} = \frac{1}{|K_a|} \sum_{k \in K_a} c_1 \exp\left(-c_2 (M_k^i - \gamma_c)^2\right),
% \end{equation}
\begin{equation}
\mathcal{L}_{C} =
% \frac{1}{|K_a| \, n}
\sum_{i=1}^{n}
\sum_{k \in K_a}
\begin{cases}
c_1 \cdot \exp\left(-c_2 (M_k^i - 1)^2\right),
& \text{if } M_k^i < 2, \\[10pt]

\left(M_k^i - 2\right)^2,
& \text{if } M_k^i \ge 2,
\end{cases}
\label{eq:conn-loss}
\end{equation}
where \( c_1 \) and \( c_2 \) are tuning parameters. The bell-shaped penalty is maximized when \( M_k^i \approx 1 \), disfavoring singleton auxiliary variables.

\textit{Digitization Loss:} Optionally, to encourage \( M_k^i \) to be an integer (e.g., 0, 1, or 2), we introduce a regularization term:
\begin{equation}
    \mathcal{L}_{\text{D}} = \frac{1}{|K|} \sum_{k \in K} \frac{1}{2} \left(1 - \sin\left(2\pi M_k^i + \frac{\pi}{2}\right)\right),
\end{equation}
where \( K \) is the set of all variable indices. While this periodic and non-convex loss can improve interpretability, it may introduce instability if over-weighted. Therefore, it should be used with a small coefficient.

Combining the semantic and syntactic objectives, the total loss is formulated as:
\begin{equation}
    \mathcal{L} = \mathcal{L}_{\text{BCE}} + \lambda_{\text{E}} \mathcal{L}_{\text{E}} + 
    \lambda_{\text{S}} \mathcal{L}_{\text{S}} + \lambda_{\text{R}} \mathcal{L}_{\text{R}} + \lambda_{\text{C}} \mathcal{L}_{\text{C}} + \lambda_{\text{D}} \mathcal{L}_{\text{D}},
\label{eq:total-loss}
\end{equation}
where \( \lambda_{\text{E}} \), \( \lambda_{\text{S}} \), \( \lambda_{\text{R}} \), \( \lambda_{\text{C}} \), and \( \lambda_{\text{D}} \) are scalar weights controlling the influence of each regularization term. As detailed in Appendix~\ref{app:loss-scheduling}, the loss weights are scheduled during training based on the role of each loss component. The loss scheduling technique gradually introduces or reweights loss components during training, allowing the model to first learn stable, high-confidence structures before enforcing stricter semantic and syntactic constraints. By dynamically adjusting loss weights over training epochs, it prevents early over-regularization and promotes more robust, interpretable convergence. Given the total loss, optimization proceeds via gradient-based methods on a given batch of data. During training, the network minimizes \( \mathcal{L} \) through backpropagation, progressively refining its weights to discover valid subrules that collectively define the same-head logical rule.
% \textbf{Curriculum Learning.}
% To extract multiple subrules from data, ANDRE employs a curriculum learning strategy that discovers subrules sequentially in order of confidence. This staged training process allows the model to prioritize simpler or more confident subrules before addressing more complex ones, thereby improving both stability and interpretability in rule extraction.
% Algorithm~\ref{andre_algo} (Appendix~\ref{ap:curriculum-learning}) outlines the complete curriculum learning procedure used by ANDRE. Given a dataset, the model is first initialized with a single subrule. The weights associated with this subrule are optimized by minimizing the total loss over multiple epochs and restarts. The best-performing set of weights across restarts is retained. The model is then extended by adding a new subrule. Previously learned weights are frozen to prevent forgetting, and only the newly added subrule is trained while earlier subrules remain fixed. This process is repeated iteratively: at each stage, the model is augmented with one new subrule, and only the latest subrule undergoes training. The procedure continues until either a predefined maximum number of subrules \( R_{\max} \) is reached, or the accuracy ratio \( a_r \) exceeds a threshold \( \tau \), triggering early stopping.

ANDRE assumes no knowledge of the number of rules in the target dataset, thereby running iteratively with different number of rules to find the most accurate model, even though it is capable of finding rules with only one-shot training. The accuracy metric used for evaluation is defined as the proportion of correctly predicted labels over all examples. At each stage, multiple restarts are performed to avoid poor local minima, ensuring that the most effective subrule is selected. This incremental process not only enhances the model convergence but also enables subrules to be naturally ranked by their learned confidence. 
After training, it is essential to extract the effective body predicates from the probabilities learned by ANDRE. Appendix~\ref{app:pred-identification} describes the criterion we use to identify all body predicates.

% To summarize, the proposed approach reformulates neuro-symbolic ILP as an attention-based conjunction-disjunction architecture operating over a continuous rule space. Each body predicate is governed by a compact trainable distribution over three symbolic forms—positive, negated, or excluded—enabling flexible and interpretable subrule construction. ANDRE introduces tailored loss functions that not only guide semantic learning for accurate prediction but also enforce syntactic compliance with logical rule structures. Furthermore, a progressive curriculum learning strategy is employed to iteratively extract multiple subrules for the same-head predicate, enhancing both stability and generalization. This fully differentiable design facilitates efficient gradient-based training while preserving interpretability through symbolic post-processing, offering a unified, scalable, and interpretable solution to differentiable ILP.

\section{Experimental Results}
\label{sec:experiments}
This section presents a comprehensive evaluation of ANDRE’s performance in learning interpretable logical rules. All experiments were implemented using PyTorch and ADAM optimizer on an NVIDIA GeForce RTX 5090. The default hyperparameter settings are provided in Table~\ref{tab:andre_params} (Appendix~\ref{app:andre-params}) but may vary slightly per dataset. To ensure stable optimization, batch training with label balancing, gradient clipping, and learning rate decay were applied across all training phases. Moreover, ANDRE model is completely vectorized and GPU-friendly, having the ability to train on large-scale datasets rapidly. 

Prior to the core experiments, two illustrative case studies were conducted on synthetic datasets in Appendix~\ref{ap:visual-example} to visualize how ANDRE updates its predicate probabilities, how loss and accuracy evolve over training, and how symbolic rules are extracted. For \texttt{grandparent} task, figure~\ref{fig:grandparent-weights} illustrates the final learned subpredicate weights. Accordingly, after applying the predicate identification procedure described in Appendix~\ref{app:pred-identification}, ANDRE extracts the following symbolic rule set:
\begin{equation*}
\begin{aligned}
\texttt{grandparent}(X_1, X_3) \leftarrow {} &
\texttt{mother}(X_1, X_2) \wedge \texttt{mother}(X_2, X_3), \\
\texttt{grandparent}(X_1, X_3) \leftarrow {} &
\texttt{father}(X_1, X_2) \wedge \texttt{father}(X_2, X_3), \\
\texttt{grandparent}(X_1, X_3) \leftarrow {} &
\texttt{mother}(X_1, X_2) \wedge \texttt{father}(X_2, X_3), \\
\texttt{grandparent}(X_1, X_3) \leftarrow {} &
\texttt{father}(X_1, X_2) \wedge \texttt{mother}(X_2, X_3),
\end{aligned}
\end{equation*}
which are both semantically and syntactically correct. Variables $X_1, X_3$ are head variables, while $X_2$ is an auxiliary variable. More details are deferred to Appendix~\ref{ap:visual-example}.
\begin{figure}[t]
    \centering
    \includegraphics[width=\linewidth]{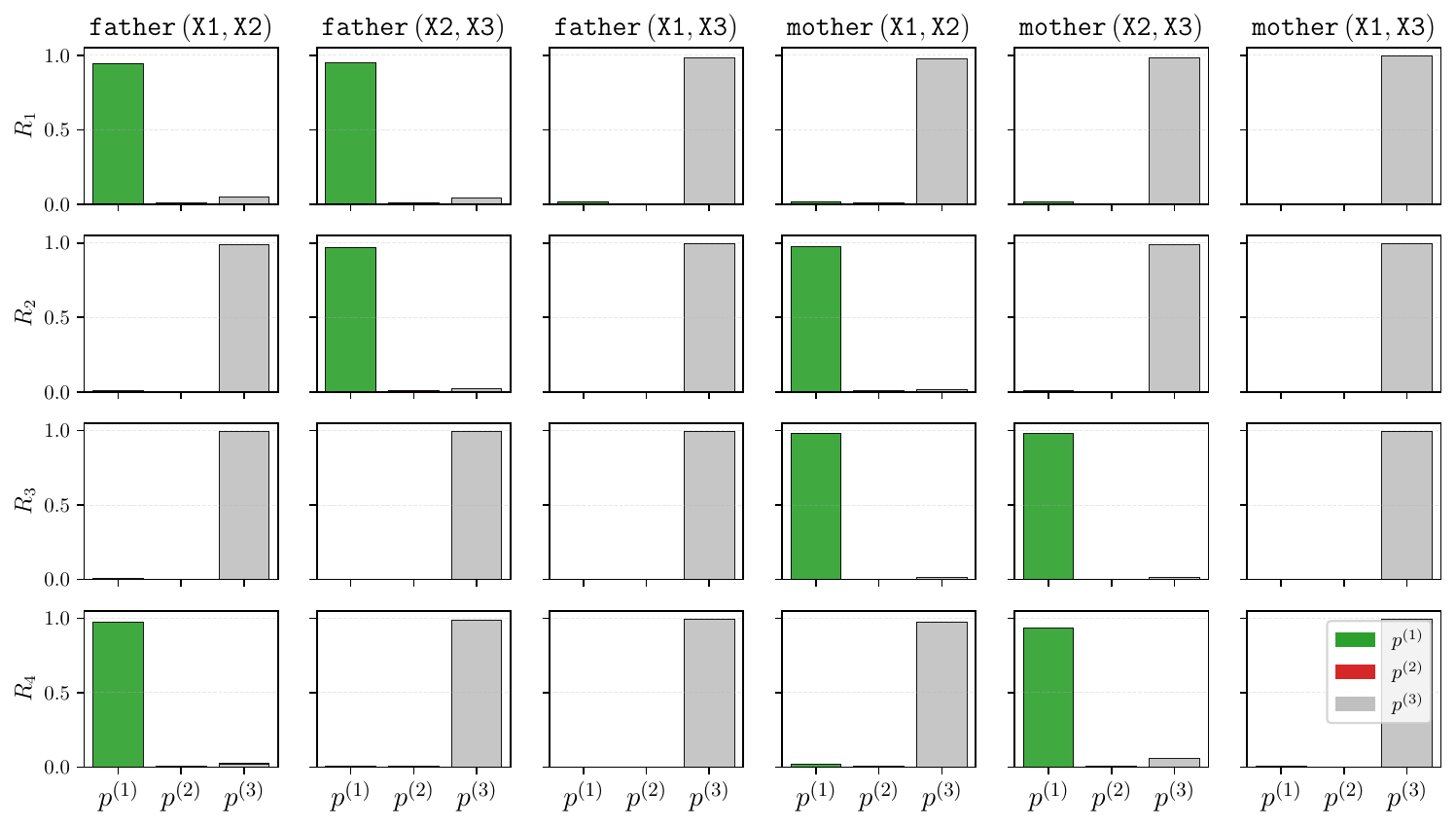}
    % \fbox{\parbox[c][5cm][c]{0.95\linewidth}{\centering Final learned ANDRE subpredicate weights}}
    \caption{Final softmax-normalized subpredicate probabilities for the Grandparent task.
    Each subrule converges to a sparse and interpretable predicate configuration. Green, red, and gray bars indicate the probability of inclusion, negation, or exclusion of each body predicate in each subrule $R_i$.}
    \label{fig:grandparent-weights}
\end{figure}

The experiments are organized into three main categories: 

\begin{wraptable}{r}{0.48\textwidth}
\centering
\scriptsize
\vspace{-5mm}
\caption{The results on ILP datasets. The symbols \cmark, $*$, and $-$ indicate that the accuracy of the generated logic program is equal to 100\%, less than 100\%, and equal to 0\%, respectively \citep{gao2024differentiable}.}
\label{tab:ilp_comparison}
\renewcommand{\arraystretch}{1.2}
\begin{tabular}{lcccc}
\toprule
\textbf{Task} & $\partial$ILP & NeuralLP & DFORL & ANDRE \\
\toprule
Predecessor       & \cmark & \cmark & \cmark & \cmark\\
Odd               & \cmark & $-$          & \cmark & \cmark\\
Even              & $-$          & $-$          & \cmark & \cmark\\
LessThan          & \cmark & \cmark & \cmark & \cmark\\
\midrule
Son               & \cmark & $*$          & \cmark & \cmark\\
Grandparent       & \cmark & \cmark & \cmark & \cmark\\
Related       & \cmark & $*$          & \cmark & \cmark\\
Father            & $-$          & $-$          & \cmark & \cmark\\
\midrule
Direct Edge     & \cmark & $*$          & \cmark & \cmark \\
Connected     & \cmark & $*$          & \cmark & \cmark \\

\toprule
\end{tabular}
\vspace{-3mm}
\end{wraptable}
\vspace{-3mm}

\subsection{Classical ILP Datasets}
To assess the generalizability and interpretability of ANDRE in real-world symbolic settings, we first evaluate its performance on classical ILP benchmarks. Similar to prior methods such as $\partial$ILP~\citep{evans2018learning}, NeuralLP~\citep{yang2017differentiable}, and DFORL~\citep{gao2024differentiable}, ANDRE is tested on over ten widely used ILP tasks. %These tasks span diverse domains, including arithmetic reasoning, family tree inference, and graph connectivity, and serve to benchmark both scalability and robustness under crisp, symbolic input conditions. 
As these benchmarks are originally unstructured, a preprocessing step is necessary to transform them into a form compatible with differentiable learning. Therefore, we apply a propositionalization procedure introduced by \citet{gao2022learning}, that converts the logical relationships into a supervised dataset $E$, consisting of $m$ predicate valuations and a corresponding Boolean target output. % Our approach extends the transformation strategy introduced by \citep{gao2022learning}, while allowing for a broader range of predicate arities beyond unary and binary forms.

Following this conversion, ANDRE is tasked with discovering a logical rule that best explains each dataset while preserving accuracy over data. Table~\ref{tab:ilp_comparison} summarizes the tasks and the performance of ANDRE compared to the baselines. Each dataset is described in full detail in Appendix~\ref{app:classical-ilp}, including body predicates, background knowledge, and sets of positive and negative examples.
ANDRE achieves $100\%$ rule extraction accuracy across all ten tasks, matching DFORL and outperforming both $\partial$ILP and NeuralLP. These results underscore ANDRE's effectiveness in symbolic environments and confirm its capacity to induce interpretable logical rules from examples. Moreover, the subrules extracted by ANDRE for each task consistently adhered to the two syntactic constraints, demonstrating the model’s built-in capability for learning rule structures.

\begin{wraptable}{r}{0.48\textwidth}
\vspace{-4mm}
\centering
\scriptsize
\caption{Performance comparison across knowledge bases. ACC@$h$ means the accuracy of the model on a dataset with head predicate $h$. Predicates \textit{blo}, \textit{int}, \textit{neg}, and \textit{intw} stand for \textit{blockpositionindex}, \textit{intergovorgs3}, \textit{negativecomm}, and \textit{interacts\_with}.}
\label{tab:benchmark_results}
\begin{tabular}{llcccc}
\toprule
\rotatebox{90}{\textbf{Dataset}} &
\rotatebox{90}{\textbf{Metrics}} &
\rotatebox{90}{\textbf{NTP$\lambda$}} &
\rotatebox{90}{\textbf{NeuralLP}} &
\rotatebox{90}{\textbf{DFORL}} &
\rotatebox{90}{\textbf{ANDRE}} \\
\toprule

% ---------------- Countries ----------------
\multirow{4}{*}{\rotatebox{90}{Countries}}
 & ACC@S1 & \textbf{100.0} & \textbf{100.0} & \textbf{100.0} & \textbf{100.0} \\
 & ACC@S2 & \textbf{100.0} & \textbf{100.0} & \textbf{100.0} & \textbf{100.0} \\
 & ACC@S3 & \textbf{100.0} & -- & \textbf{100.0} & \textbf{100.0} \\
 & & & & & \\
\midrule

% ---------------- Nations ----------------
\multirow{7}{*}{\rotatebox{90}{Nations}}
 & MRR      & 0.418 & 0.565 & 0.789 & \textbf{0.831} \\
 & HITS@1   & 41.79 & 52.49 & 73.88 & \textbf{79.52} \\
 & HITS@3   & 41.79 & 60.95 & 84.58 &  \textbf{84.73} \\
 & HITS@10  & 41.79 & 61.19 & 85.07 & \textbf{93.81} \\
 & ACC@blo  & \textbf{100.00} & 50.00 & \textbf{100.00} & 99.06 \\
 & ACC@int  & 84.62 & 84.62 & 84.62 & \textbf{89.38} \\
 & ACC@neg  & 37.50 & 75.00 & 75.00 & \textbf{89.24} \\
\midrule

% ---------------- UMLS ----------------
\multirow{6}{*}{\rotatebox{90}{UMLS}}
 & MRR      & 0.301 & 0.667 & 0.750 & \textbf{0.772} \\
 & HITS@1   & 29.95 & 61.27 & \textbf{71.41} & 69.28 \\
 & HITS@3   & 30.11 & 72.31 & 78.82 & \textbf{81.19} \\
 & HITS@10  & 30.11 & 72.31 & 78.97 & \textbf{83.36} \\
 & ACC@isa  & 65.96 & 63.83 & 91.48 & \textbf{94.53} \\
 & ACC@intw & 83.67 & 86.67 & \textbf{100.0} & 98.97 \\
% \midrule

% ---------------- Kinship ----------------
% \multirow{6}{*}{\rotatebox{90}{Kinship}}
%  & MRR        & 0.544 & 0.401 & \textbf{0.876} &  \\
%  & HITS@1     & 52.51 & 40.13 & \textbf{83.15} &  \\
%  & HITS@3     & 56.33 & 40.13 & \textbf{92.27} & 91.02 \\
%  & HITS@10    & 56.33 & 40.13 & 92.36 & \textbf{98.37} \\
%  & ACC@term15 & 91.21 & 83.52 & 95.86 & \textbf{97.37} \\
%  & ACC@term16 & 92.37 & 88.55 & 96.95 & \textbf{97.05} \\
\bottomrule
\end{tabular}
\vspace{-2mm}
\end{wraptable}

\subsection{Large-Scale Knowledge Bases}
In this section, we test ANDRE on three large-scale knowledge bases, including Countries dataset~\citep{bouchard2015approximate}, Unified Medical Language System (UMLS), and Nations dataset~\citep{Kok2007StatisticalPI}. Following~\citet{gao2022learning}, the Countries dataset is divided into three different subcategories ordered by learning difficulty: S1, S2, and S3 sub-datasets. Statistical details of the dataset are deferred to Appendix~\ref{app:andre-knowledge-base}. Following the propositionalization step, a trainable dataset is extracted from each knowledge base, based on the target head predicate. Then, each dataset is split into a train and a validation set. After training, three different performance metrics are computed: Accuracy, Mean Reciprocal Ranks (MRR), and HITS@K, where $\text{K} \in \{1, 3, 10\}$.

Table~\ref{tab:benchmark_results} compares the performance of ANDRE with baselines, including NTP$\lambda$~\citep{Kok2007StatisticalPI}, NeuralLP, and DFORL. Based on all metrics, ANDRE outperforms the baselines in all datasets, except Nations - \textit{blo} and UMLS - \textit{intw} datasets. In general, the performance of ANDRE is either comparable or better than DFORL as the strongest baseline, confirming the reliability and scalability of ANDRE. The extracted rules for each dataset and predicate are presented in Appendix~\ref{app:andre-knowledge-base}. Moreover, we computed and compared the training and evaluation time for ANDRE and the baselines in Table~\ref{tab:runtime_comparison} (Appendix~\ref{tab:train-time}), which shows that ANDRE is significantly faster than DFORL, though it is not the fastest method. Based on Appendix~\ref{ap:further-results}, ANDRE also outperformed the baselines on the UW-CSE~\citep{uwcse2005} dataset and has shown competitive performance on the Alzheimer-amine dataset compared to the strongest baseline.

\subsection{Synthetic Probabilistic Data}
As ANDRE is specifically designed to handle probabilistic (non-Boolean) data, classical ILP benchmarks—with their crisp symbolic representations—are not ideal for fully evaluating its capabilities. To address this, we construct synthetic datasets tailored to assess ANDRE’s ability to learn interpretable logical programs under uncertainty. Unlike traditional ILP benchmarks, these datasets allow controlled manipulation of data properties such as noise, ambiguity, and rule complexity, thereby providing a more rigorous and scalable testbed for neuro-symbolic methods like ANDRE.

We generated a diverse collection of synthetic datasets, each corresponding to a logical rule of varying complexity and number of subrules. All data points were sampled from a uniform distribution and split into training and validation subsets using the same underlying distribution to ensure consistency. A fixed set of random seeds was used throughout all experiments for reproducibility. Each dataset consists of the probabilistic valuations of $m$ candidate body predicates with the corresponding Boolean label derived from a target logical rule (see Appendix~\ref{ap:visual-example} for an example of how $h$ is determined).

We generated three different datasets, based on the following hierarchical rules with increasing complexity: $\texttt{R}_1 \leftarrow \texttt{b}_1 \wedge \neg \texttt{b}_9$, $\texttt{R}_2 \leftarrow \texttt{R}_1 \vee (\neg \texttt{b}_2 \wedge \texttt{b}_8)$, and $\texttt{R}_3 \leftarrow \texttt{R}_2 \vee (\texttt{b}_3 \wedge \neg \texttt{b}_5 \wedge \texttt{b}_7)$. In each case, the model must infer the correct rules from probabilistic input features only with semantic losses. We only benchmark ANDRE against DFORL, as the only publicly available and most competitive method on classical ILP datasets and knowledge bases.

% The corresponding results are enumerated in two tables: Table~\ref{tab:andre_vs_dforl_full}  as the extended version and the other as the compressed one.
Based on Table~\ref{tab:andre_vs_dforl_all}, ANDRE consistently outperformed DFORL by $21.63\%$ on average. According to the extended Table~\ref{tab:andre_vs_dforl_full} in Appendix~\ref{ap:results}, ANDRE successfully extracts the target logical program when provided with a sufficient number of labeled examples. %For instance, on rules $\texttt{R}_1$, $\texttt{R}_2$, and $\texttt{R}_3$, ANDRE fails to extract the correct rule when using $20$, $50$, and $200$ samples respectively, despite achieving average training and test accuracies of $0.91$ and $0.82$. In these cases, the model exhibits mild overfitting due to limited data. DFORL performs worse, with corresponding averages of $0.78$ and $0.78$, and similarly fails to extract the target rules.
% As the number of samples increases, ANDRE performs even better. 
%Across the remaining experiments with larger datasets, it achieves average training and test accuracies of $0.94$ and $0.91$, successfully extracting all target rules. DFORL, by contrast, records average accuracies of $0.74$ and $0.74$ and extracts only a fraction of the rules. 
% These results affirm that while ANDRE is data-driven, its performance scales reliably with dataset size. As a practical guideline, we recommend using at least $10$ samples per potential predicate for effective rule extraction.
%From the perspective of rule complexity, ANDRE remains robust: as complexity increases from $\texttt{R}_1$ to $\texttt{R}_3$, its average test accuracy decreases only slightly from $0.92$ to $0.90$. DFORL, in contrast, exhibits a sharper decline from $0.845$ to $0.697$, due to the reliance on the inaccurate product-based conjunctions. The drop in accuracy for DFORL is $5.92\times$ greater than that observed for ANDRE, highlighting the latter's resilience to rule complexity. Furthermore, ANDRE successfully extracts subrules in $8/11$ experiments, compared to only $2/11$ for DFORL, demonstrating superior structural extraction under uncertainty.
From the perspective of rule complexity, ANDRE remains robust: as complexity increases from $\texttt{R}_1$ to $\texttt{R}_3$, the drop in accuracy for DFORL is $5.92\times$ greater than that observed for ANDRE, highlighting the latter's resilience to rule complexity. %Furthermore, based on the extended table, ANDRE successfully extracts subrules in $8/11$ experiments, compared to only $2/11$ for DFORL, demonstrating superior rule extraction on non-Boolean data.
\begin{wraptable}{r}{0.5\textwidth}
\vspace{1mm}
\centering
\scriptsize
\caption{Validation accuracy on complex synthetic datasets with varying number of subrules (compressed version of Table~\ref{tab:andre_vs_dforl_full} and \ref{tab:andre_vs_dforl_noisy}).}
\vspace{0.1in}
\renewcommand{\arraystretch}{1.2}
% \scriptsize
\begin{tabular}{lcccc}
\toprule
\multirow{2}{*}{\centering\makecell{\textbf{Dataset}}} &
\multirow{2}{*}{\centering\makecell{\textbf{Sample}\\\textbf{Size Range}}} & \multirow{2}{*}{\centering\makecell{\textbf{Noise }\\\textbf{Range}}} &
\multicolumn{2}{c}{\textbf{Average Accuracy}} \\
\cline{4-5}
& & &
\textbf{DFORL} &
\textbf{ANDRE} \\
% \cline{4-7}
% & & & \textbf{Train} & \textbf{Test} & \textbf{Train} & \textbf{Test} & \textbf{ANDRE} & \textbf{DFORL} \\
\midrule
$\texttt{R}_1$ & $20 - 200$ & $-$  & $0.85$ & \textbf{0.92}\\
% \hline
$\texttt{R}_2$ & $50 - 1000$ & $-$  & $0.73$ & \textbf{0.94} \\
% \hline
$\texttt{R}_3$ & $200 - 2000$ & $-$  & $0.70$ & \textbf{0.90} \\
\midrule
$\texttt{R}_4$ & $200$ & $10 - 30$  & $0.61$ & \textbf{0.71} \\
% \hline
$\texttt{R}_5$ & $500$ & $5 - 45$  & $0.52$ & \textbf{0.71} \\
% \hline
$\texttt{R}_6$ & $1000$ & $5 - 25$  & $0.48$ & \textbf{0.75} \\
\toprule
\end{tabular}
\label{tab:andre_vs_dforl_all}
\vspace{-3mm}
\end{wraptable}

\textbf{Noisy Synthetic Data.}
To assess ANDRE's robustness in the presence of noisy labels, we conducted a series of experiments using probabilistic datasets with intentionally corrupted labels. In these datasets, a controlled percentage of both positive and negative labels were randomly flipped, simulating real-world scenarios with mislabeled data. While the complexity of the target rules was kept consistent across experiments, the number of training samples was large and fixed for each rule to eliminate underfitting as a confounding factor. %As in the previous section, DFORL serves as the baseline for comparison under identical noisy conditions.

The target rules used in this study are structurally similar to those in the previous experiments but drawn from a separate hierarchy: $\texttt{R}_4 \leftarrow \neg \texttt{b}_1 \wedge \texttt{b}_9$, $\texttt{R}_5 \leftarrow \texttt{R}_4 \vee (\texttt{b}_2 \wedge \neg \texttt{b}_8)$, $\texttt{R}_6 \leftarrow \texttt{R}_5 \vee (\neg \texttt{b}_3 \wedge \texttt{b}_5 \wedge \neg \texttt{b}_7)$.
Table~\ref{tab:andre_vs_dforl_noisy} (Appendix~\ref{ap:results}) reports the performance of ANDRE and DFORL across varying levels of label noise, while Table~\ref{tab:andre_vs_dforl_all} compresses the results. %As expected, the accuracy of both models declines as noise increases. However, ANDRE exhibits significantly greater resilience: it successfully extracts the correct rule structure for \( \texttt{R}_4 \) and \( \texttt{R}_5 \) under noise levels as high as $25$\%, and for \( \texttt{R}_6 \) up to $15$\% noise. 
On the synthetic probabilistic benchmarks (\( \texttt{R}_1 \)–\( \texttt{R}_6 \)), ANDRE consistently achieves higher predictive accuracy and substantially improved rule extraction quality compared to DFORL. In particular, while DFORL collapses under moderate label noise due to the use of product-based fuzzy operators, ANDRE remains stable, correctly recovering 9 out of 14 ground-truth rules across noisy settings. Although performance degrades for the deepest rule (\( \texttt{R}_6 \)) at the highest noise level, ANDRE still extracts significantly more correct symbolic structure than competing differentiable ILP methods. These results empirically confirm that attention-based operators preserve dominant predicates under uncertainty, enabling more reliable rule induction from probabilistic data.
% In contrast, DFORL fails to extract the rule in nearly all noisy settings. In total, ANDRE successfully extracted the correct subrules in $9/14$ noisy experiments, compared to only $1/14$ for DFORL. 
% These results underscore ANDRE’s exceptional robustness to moderate levels of label corruption—an essential property for real-world applications where supervision is often noisy or imperfect. This resilience stems directly from the flexibility of ANDRE’s rule space and the accuracy of its attention-based operators, which are well-suited to reasoning over probabilistic data.

\textbf{Ablation Study.}
To assess the impact of the attention-based logical network on ANDRE's performance, we conducted an ablation study by comparing the standard ANDRE model to a variant in which the attention-based conjunction and disjunction operators were replaced with product-based alternatives. The experimental setup mirrors the evaluations in Table~\ref{tab:andre_vs_dforl_full} (Appendix~\ref{ap:results}), and we report the average test accuracy for both configurations. As shown in Table~\ref{tab:ablation}, the attention-based operators improve accuracy by $21.05\%$, demonstrating their critical role in effective rule learning.
\vspace{-1mm}
\begin{wraptable}{r}{0.45\textwidth}
\centering
\scriptsize
\caption{Ablation Study Results.}
\renewcommand{\arraystretch}{1.2}
\begin{tabular}{lcc}
\toprule
\textbf{Method} & \textbf{ANDRE w/o Attention} & \textbf{ANDRE} \\
\hline
\textbf{Accuracy} & $0.76$ & $0.92$ \\
\toprule
\end{tabular}
\vspace{-4mm}
\label{tab:ablation}
\end{wraptable}
\vspace{-2mm}

Overall, the strong performance of ANDRE arises from its carefully designed architecture. The introduced rule space allows flexible inclusion or exclusion of body predicates, even in the presence of ambiguity. Unlike product-based operators, ANDRE’s attention-based conjunction and disjunction mechanisms maintain full differentiability, preserve accuracy, and alleviate the vanishing gradient problem common in fuzzy logic reasoning. Additionally, the syntactic loss components encourage syntactically valid rule structures, even when semantic signals are noisy or degraded. %Finally, ANDRE’s curriculum learning strategy enables progressive rule extraction, which mitigates error propagation from earlier subrules and enhances robustness under noisy supervision.

% Taken together, these findings highlight ANDRE’s contributions as a neuro-symbolic ILP framework that combines interpretability, robustness, and generalization. Its ability to reliably extract symbolic rules from noisy, probabilistic data positions it as a compelling tool for domains such as sensor fusion, healthcare diagnostics, and scientific discovery—where uncertainty and label imperfections are common.

% ANDRE has a bright future: (1) ANDRE has the ability to be connected to the regular NNs serially, helping to improve the reasoning and transparency in dense models by proposing a fully tractable network. (2) This model can be used as a policy network in reinforcement learning. (3) Moreover, the weights of the model can be manually set and fixed in order to impose specific background knowledge or logic-based constraints into the network.
\vspace{-3mm}

\section{Conclusion}
\label{sec:conclusion}
We introduced an \textit{Attention-based Neuro-symbolic Differentiable Rule Extractor (ANDRE)}, a novel neuro-symbolic ILP method that formulates rule learning as optimization over a continuous rule space. By leveraging soft attention mechanisms for conjunction and disjunction, ANDRE enables differentiable learning of symbolic rules without requiring predefined rule templates. The model identifies the most relevant predicates through a tailored learning strategy and enforces logical consistency using specific syntactic loss functions. Experimental results across classical ILP benchmarks, knowledge bases, and synthetic probabilistic datasets demonstrated ANDRE’s strong performance in both rule extraction and classification accuracy. Notably, ANDRE remains robust in the presence of noise, successfully extracting symbolic rules under up to $25$\% mislabeled data depending on the task complexity. Compared to state-of-the-art data-driven ILP methods, ANDRE consistently outperforms in terms of accuracy, interpretability, and stability. These findings highlight ANDRE’s potential as a reliable and interpretable component in neuro-symbolic learning pipelines. %Its ability to extract structured rules under uncertainty makes it particularly well-suited for real-world applications that demand integration of logical reasoning with gradient-based learning, including scientific discovery, robotics, and decision support systems.

\paragraph{Limitations and Future Research.}
Despite its strong performance, ANDRE has several limitations. First, while the proposed continuous rule space enables flexible predicate selection, the current architecture is primarily designed for relatively shallow rule structures and may struggle to efficiently represent more complex or deeply nested logical compositions. Second, the model relies on propositionalization of relational data, which may limit its ability to fully exploit rich relational structure in large knowledge graphs. Finally, symbolic rule extraction depends on confidence thresholds over learned distributions, which may introduce sensitivity in borderline cases. Future work will focus on extending the framework to support more expressive rule compositions and developing more principled mechanisms for direct relational reasoning without full propositionalization.

\bibliography{iclr2025_conference}
\bibliographystyle{iclr2025_conference}

\appendix
\section*{Appendix}

\section{Justification for Prob-Sum Activation Function}
\label{ap:prob-sum-justification}

To identify the most confident subpredicate among three possible candidates, we employ Eq.~\ref{prob_sum}, which serves as a smooth and differentiable approximation of logical disjunction over fuzzy truth values. This formulation is grounded in fuzzy logic principles, particularly the product $t$-norm and De Morgan’s laws.

In fuzzy logic, disjunctions are typically modeled using t-conorms. One such t-conorm is derived from the product t-norm, where conjunction is approximated as \( x \wedge y \approx xy \). Applying De Morgan’s law, the disjunction can be expressed as:
\begin{equation*}
    \mu_{x \lor y} = 1 - (1 - \mu_x)(1 - \mu_y),
\end{equation*}
where \( \mu_x \) and \( \mu_y \) denote the fuzzy membership values of inputs \( x \) and \( y \), respectively. This formulation naturally extends to three variables:
\begin{equation*}
    \mu_{x \lor y \lor z} = 1 - (1 - \mu_x)(1 - \mu_y)(1 - \mu_z).
\end{equation*}

Letting \( \mu_x = x \), \( \mu_y = y \), and \( \mu_z = z \), we expand the expression:
\begin{equation*}
\begin{aligned}
    f(x, y, z) &= \mu_{x \lor y \lor z} \\
               &= 1 - (1 - x)(1 - y)(1 - z) \\
               &= 1 - \left(1 - x - y - z + xy + xz + yz - xyz \right) \\
               &= x + y + z - xy - xz - yz + xyz.
\end{aligned}
\end{equation*}

This derivation confirms that \( f(x, y, z) \) corresponds exactly to the fuzzy disjunction defined by the product t-conorm extended to three variables. The function captures the cumulative confidence across all subpredicates while attenuating redundancy via subtraction of pairwise interactions and reinforcing agreement via the \( +xyz \) term.

By design, this function produces higher values when at least one input is close to 1.0, thereby emphasizing the most confident subpredicate. Since our objective is to identify the subpredicate with the highest contribution to the overall rule activation, this function offers a principled, differentiable, and interpretable measure rooted in fuzzy logic semantics.

\section{Attention-based vs. Product-based Operators}
\label{ap:attention-product}
\subsection{Softmin Attention vs. Product $t$-norm}
\begin{lemma}
Let $\mathbf{b} = (b_1, b_2, \dots, b_m) \in \mathbb{R}^m$ and let $b_{\min} = \min_i b_i$. Then for any $\beta > 0$, the softmin operator
\[
T_{\text{attention}}(\beta) = \sum_{i=1}^m \alpha_i(\beta) \cdot b_i, \quad \text{where } \alpha_i(\beta) = \frac{e^{-\beta b_i}}{\sum_{j=1}^m e^{-\beta b_j}},
\]
satisfies:
\[
T_{\text{attention}}(\beta) \geq b_{\min},
\]
with equality if and only if $b_i = b_{\min}$ for all $i$ or in the limit $\beta \to \infty$.
\end{lemma}

\begin{proof}
Each $\alpha_i(\beta)$ is strictly positive and the weights form a convex combination: $\sum_i \alpha_i(\beta) = 1$ and $\alpha_i(\beta) > 0$. Let $b_{\min} = b_{k}$ for some index $k$. Then:
\[
T_{\text{attention}}(\beta) = \sum_{i=1}^m \alpha_i(\beta) \cdot b_i = b_{\min} \sum_{i=1}^m \alpha_i(\beta) + \sum_{i=1}^m \alpha_i(\beta) \cdot (b_i - b_{\min}).
\]
Since each $b_i - b_{\min} \geq 0$ and at least one $b_i > b_{\min}$ (unless all values are equal), the second term is strictly positive. Therefore:
\[
T_{\text{attention}}(\beta) > b_{\min} \quad \text{unless } b_i = b_{\min} \text{ for all } i.
\]
In the limit $\beta \to \infty$, the softmin attention converges to the minimum:
\[
\lim_{\beta \to \infty} T_{\text{attention}}(\beta) = b_{\min}.
\]
\end{proof}

\begin{theorem}
Let $\mathbf{b} = (b_1, b_2, \dots, b_m) \in (0, 1]^m$ be a vector of probabilistic values, and let $b_{\min} = \min_i b_i$. Define the following:
\begin{itemize}
    \item Softmin attention: \\ \[ T_{\text{attention}}(\beta) = \sum_{i=1}^m \alpha_i(\beta) \cdot b_i \quad \text{where} \,\,\alpha_i(\beta) = \dfrac{e^{-\beta b_i}}{\sum_{j=1}^m e^{-\beta b_j}} \]
    \item Product t-norm: \\ \[T_{\text{prod}} = \prod_{i=1}^m b_i\]
\end{itemize}
Then, there exists a finite $\beta^* > 0$ such that the approximation error of the softmin attention operator is strictly less than that of the product t-norm:
\[
\epsilon_{\text{attention}}(\beta^*) = T_{\text{attention}}(\beta^*) - b_{\min} < b_{\min} - T_{\text{prod}} = \epsilon_{\text{prod}}.
\]
\end{theorem}

\begin{proof}
Since all $b_i \in (0, 1]$, the product t-norm satisfies $T_{\text{prod}} < b_{\min}$ (unless all $b_i = 1$), so:
\[
\epsilon_{\text{prod}} = b_{\min} - T_{\text{prod}} > 0.
\]

By the Lemma, $T_{\text{attention}}(\beta) > b_{\min}$ for all finite $\beta$, and:
\[
\lim_{\beta \to \infty} T_{\text{attention}}(\beta) = b_{\min} \quad \Rightarrow \quad \lim_{\beta \to \infty} \epsilon_{\text{attention}}(\beta) = 0.
\]

Because $T_{\text{attention}}(\beta)$ is a continuous and strictly decreasing function of $\beta$, and since:
\[
\epsilon_{\text{prod}} > 0 = \lim_{\beta \to \infty} \epsilon_{\text{attention}}(\beta),
\]
then by the intermediate value theorem, there exists a finite $\beta^* > 0$ such that:
\[
\epsilon_{\text{attention}}(\beta^*) < \epsilon_{\text{prod}} \quad \Leftrightarrow \quad T_{\text{attention}}(\beta^*) < 2b_{\min} - T_{\text{prod}}.
\]
Hence, softmin attention provides a strictly better approximation to $\min$ than the product t-norm for some finite value of $\beta$.
\end{proof}

\subsection{Softmax Attention vs. Product-based $s$-norm}
We now compare two fuzzy disjunction operators used to approximate the logical \textsc{or} (\(\max\)) over fuzzy values $\mathbf{b} = (b_1, b_2, \dots, b_m) \in (0,1]^m$:

\begin{itemize}
    \item Softmax Attention:
    \[
    S_{\text{attention}}(\beta) = \sum_{i=1}^m \alpha_i(\beta) \cdot b_i, \quad \text{where } \alpha_i(\beta) = \frac{e^{\beta b_i}}{\sum_{j=1}^m e^{\beta b_j}}
    \]
    \item Product-based $s$-norm (disjunction):
    \[
    S_{\text{prod}} = 1 - \prod_{i=1}^m (1 - b_i)
    \]
\end{itemize}

Both operators approximate \(\max(b_1, \dots, b_m)\), but through different mechanisms. The first is a parametric softmax that converges to the maximum as $\beta \to \infty$, while the second uses De Morgan’s law and the product $t$-norm to approximate the complement of a conjunction.

\begin{theorem}
For any fuzzy values $\mathbf{b} \in (0, 1]^m$, there exists a finite $\beta^* > 0$ such that the softmax attention approximation is strictly more accurate than the product-based $s$-norm:
\[
\epsilon_{\text{attention}}(\beta^*) = \max_i b_i - S_{\text{attention}}(\beta^*) < \max_i b_i - S_{\text{prod}} = \epsilon_{\text{prod}}.
\]
\end{theorem}

\begin{proof}[Sketch]
The product-based $s$-norm under-approximates the disjunction, since:
\[
S_{\text{prod}} = 1 - \prod_{i=1}^m (1 - b_i) < \max_i b_i,
\]
with equality only when all but one $b_i$ are zero. The softmax operator is a smooth approximation to $\max(b_i)$ from below and satisfies:
\[
\lim_{\beta \to \infty} S_{\text{attention}}(\beta) = \max_i b_i.
\]
Because $S_{\text{attention}}(\beta)$ is continuous and increasing in $\beta$ and converges to the exact max, there always exists a finite $\beta^* > 0$ such that:
\[
S_{\text{attention}}(\beta^*) > S_{\text{prod}} \quad \Rightarrow \quad \epsilon_{\text{attention}}(\beta^*) < \epsilon_{\text{prod}}.
\]
Hence, softmax attention is strictly more accurate than the product-based disjunction for sufficiently large (but finite) $\beta$.
\end{proof}

\subsection{Illustrative Comparison of Conjunction and Disjunction Operators}
To further illustrate the advantages of the proposed attention-based operators, we present a simple example comparing them with commonly used alternatives for approximating logical conjunction and disjunction.

We consider four fuzzy membership functions defined over the domain $x \in [0,1]$, denoted by $\{\mu_1(x), \mu_2(x), \mu_3(x), \mu_4(x)\}$. These functions include both linear and nonlinear (Gaussian-shaped) patterns, intentionally chosen to produce diverse and challenging aggregation scenarios.

For conjunction, we compare the true minimum $\min_i \mu_i(x)$ with three differentiable approximations: the proposed attention-based softmin operator, the product t-norm $\prod_i \mu_i(x)$, and a regularized geometric mean (RGM). For disjunction, we compare the true maximum $\max_i \mu_i(x)$ with the attention-based softmax operator, the probabilistic OR, $1 - \prod_i (1 - \mu_i(x))$, and an RGM-based variant.
\begin{figure}[h]
    \centering
    \subfloat[Max Operator]{\includegraphics[width=0.48\textwidth]{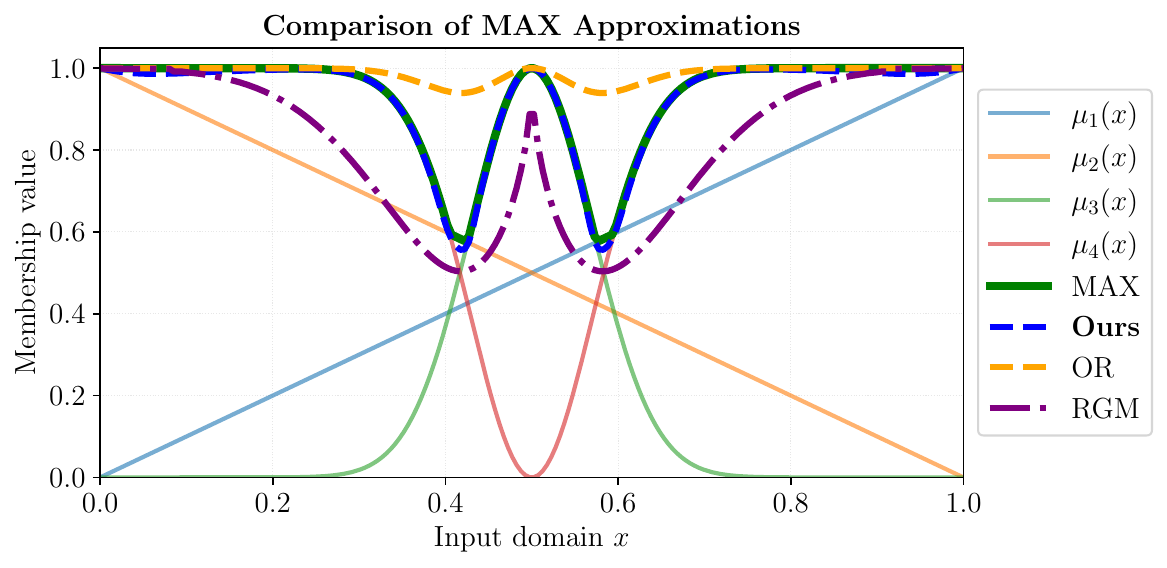}}
    \hfill
    \subfloat[Min Operator]{\includegraphics[width=0.48\textwidth]{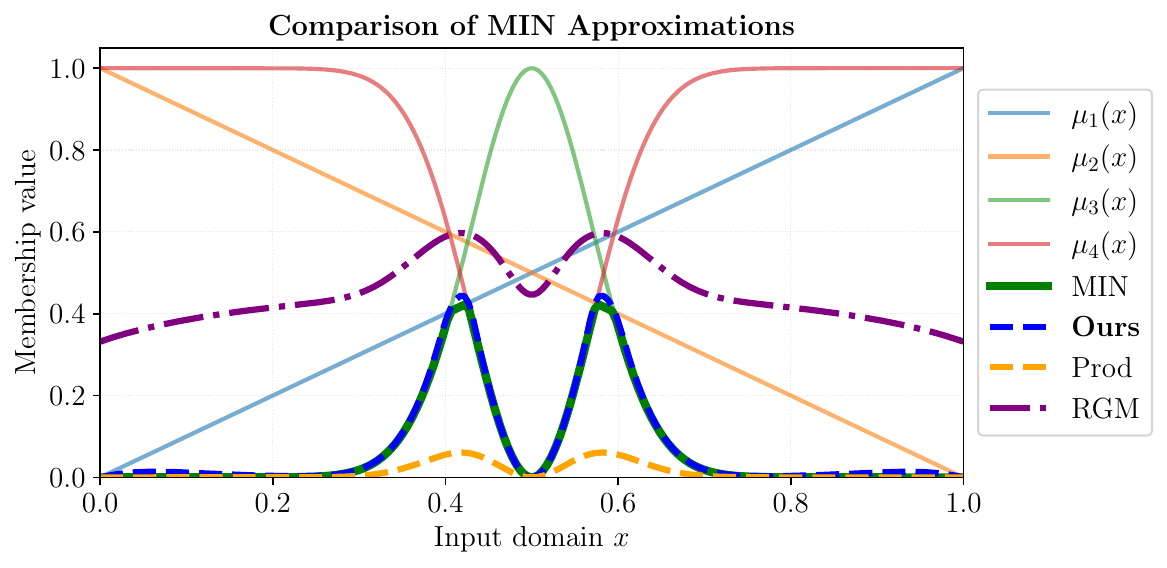}}
    \caption{Comparison of logical operators. The proposed attention-based operators closely approximate the true MIN/MAX while remaining smooth and stable.}
    \label{fig:fuzzy_operators}
\end{figure}

Figure~\ref{fig:fuzzy_operators} shows that the proposed attention-based operators closely track the true minimum and maximum across the entire domain. In the conjunction case, the product t-norm exhibits a strong bias toward zero whenever any input is small, resulting in overly conservative outputs and poor gradient behavior. The regularized geometric mean partially mitigates this effect but still deviates from the true minimum, especially in regions where inputs differ significantly. In contrast, the softmin operator adaptively concentrates weight on the smallest input, yielding a much tighter and smoother approximation.

A similar trend is observed for disjunction. The probabilistic OR operator tends to overestimate the maximum when multiple inputs take moderately large values, while the RGM-based operator introduces systematic bias due to its multiplicative structure. The proposed softmax operator, however, selectively emphasizes the largest input and provides a close and stable approximation to the true maximum.

This example highlights that the attention-based formulation achieves both higher fidelity to logical min/max behavior and improved numerical properties. In particular, it avoids the degeneracy and vanishing-gradient issues of product-based operators while preserving smoothness, making it well-suited for end-to-end differentiable rule learning.

\section{Syntactic Formats for Structured Rules}
\label{app:syntactic-formats}
In ILP, structured rules must adhere to certain syntactic constraints to ensure logical validity and semantic interpretability. ANDRE incorporates these constraints during training through syntactic loss functions (see Section~\ref{sec:methodology}). Below, we detail the two fundamental syntactic formats enforced in this framework.

\begin{itemize}[leftmargin=1.5em, itemsep=0.1em]
    \item \textbf{Head Variable Inclusion:}  
    All variables that appear in the head of the rule must also occur at least once in the body predicates. This requirement ensures that the output of the rule is logically grounded in its input facts and that no unbound variables are introduced during inference.
    
    \textit{Valid example:}
    \[
    \texttt{grandparent}(X_1, X_3)\, \texttt{:- parent}(X_1, X_2),\, \texttt{parent}(X_2, X_3).
    \]
    In this rule, the head predicate \texttt{grandparent} involves variables \( X_1 \) and \( X_3 \), both of which appear in the body predicates, satisfying the head variable inclusion constraint.

    \textit{Invalid example:}
    \[
    \texttt{grandparent}(X_1, X_3)\, \texttt{:- parent}(X_3, X_2),\, \texttt{parent}(X_2, X_3).
    \]
    Here, variable \( X_1 \) appears in the head but is missing from the body, violating the condition and making the rule syntactically invalid.

    \item \textbf{Auxiliary Variable Connectivity:}  
    Auxiliary variables—those that do not appear in the head—must either:
    \begin{enumerate}
        \item Appear in at least two distinct body predicates, forming a connection between them, or
        \item Be completely absent from the rule (i.e., not appear in any body predicates).
    \end{enumerate}
    This condition ensures that auxiliary variables are semantically meaningful and contribute to the logical relationship expressed in the rule.

    \textit{Valid example (connected auxiliary):}
    \[
    \texttt{grandparent}(X_1, X_3)\, \texttt{:- parent}(X_1, X_2),\, \texttt{parent}(X_2, X_3).
    \]
    Variable \( X_2 \) does not appear in the head and thus is auxiliary, but it occurs in both body predicates, satisfying the connectivity constraint.

    \textit{Invalid example (disconnected auxiliary):}
    \[
    \texttt{grandparent}(X_1, X_3)\, \texttt{:- parent}(X_1, X_2),\, \texttt{lives\_in}(X_4, \texttt{City}).
    \]
    Variable \( X_4 \) appears only once and does not participate in connecting any logical entities relevant to the head predicate. Such disconnected auxiliary variables weaken interpretability and are discouraged.

\end{itemize}

These two syntactic formats are not only essential for maintaining logical consistency, but they also promote generalizability and ease of interpretation—both of which are central goals of ANDRE's framework.

\section{Loss Scheduling Strategy}
\label{app:loss-scheduling}
% \addcontentsline{toc}{section}{Appendix X: Loss Scheduling Strategy}

The total training objective of ANDRE (Eq.~(\ref{eq:total-loss})) combines semantic and syntactic losses through a set of scalar coefficients
$\{\lambda_E, \lambda_S, \lambda_R, \lambda_C, \lambda_D\}$.
Rather than keeping these coefficients fixed throughout training, we adopt a \emph{progressive loss scheduling strategy}
to stabilize optimization and improve rule extraction under noisy and probabilistic supervision.

Different loss components serve distinct purposes at different stages of learning.
Early in training, ANDRE must prioritize semantic alignment with the data in order to identify promising predicate structures.
Overly strong syntactic or discretization constraints at this stage can prematurely restrict exploration and lead to poor local optima.
Conversely, in later training stages, enforcing syntactic validity and rule sparsity becomes essential for interpretability and symbolic extraction.

To address this trade-off, we gradually increase the influence of syntactic regularizers while simultaneously relaxing diversity constraints,
allowing ANDRE to transition smoothly from \emph{exploration} to \emph{structural consolidation}.

\textbf{Scheduled Loss Coefficients:} Let $t \in \{0, \dots, T\}$ denote the current training epoch, and define the normalized training progress
$\rho = t / T \in [0,1]$.
The loss coefficients are scheduled as follows.

\textit{Entropy Regularization ($\lambda_E$).}
The entropy loss (Eq.~(\ref{eq:entropy-loss})) is introduced linearly to encourage early exploration over the symbolic subpredicate space:
\begin{equation*}
\lambda_E(t) = \rho \, \lambda_E^{\max}.
\end{equation*}
This prevents premature convergence to arbitrary predicate selections and improves robustness in noisy settings.

\textit{Similarity Loss ($\lambda_S$).}
The similarity loss promotes diversity across subrules by penalizing cosine similarity between their weight vectors.
Unlike the other terms, its coefficient is \emph{decreased} during training:
\begin{equation*}
\lambda_S(t) = \lambda_S^{\max} - \rho^2 \left(\lambda_S^{\max} - \lambda_S^{\min}\right).
\end{equation*}
This allows strong diversity pressure early on, while permitting convergence toward structurally similar subrules if supported by the data.

\textit{Range-Restricted Loss ($\lambda_R$).}
The range-restricted loss (Eq.~(14)) enforces head variable inclusion.
Its coefficient follows a quadratic schedule:
\begin{equation*}
\lambda_R(t) = \frac{\rho^2}{|K_h| \cdot n} \, \lambda_R^{\max}.
\end{equation*}
This delayed activation allows the model to first discover semantically relevant predicates before enforcing strict variable coverage.

\textit{Connectedness Loss ($\lambda_C$).}
The connectedness loss (Eq.~(\ref{eq:conn-loss})) penalizes singleton auxiliary variables.
To avoid early over-penalization, its weight is also increased quadratically:
\begin{equation*}
\lambda_C(t) = \frac{\rho^2}{|K_a| \cdot n} \, \lambda_C^{\max}.
\end{equation*}
This ensures that auxiliary-variable connectivity emerges naturally from the learned rule structure.

\textit{Digitization Loss ($\lambda_D$).}
The digitization loss encourages integer-valued variable usage counts.
Because this loss is highly non-convex, it is introduced cautiously using:
\begin{equation*}
\lambda_D(t) = \rho^2 \, \lambda_D^{\max}.
\end{equation*}
In practice, $\lambda_D^{\max}$ is kept small to avoid destabilizing training.

\textit{Final Objective:}

With scheduling applied, the total loss at epoch $t$ becomes:
\begin{equation*}
\mathcal{L}(t) = \mathcal{L}_{\mathrm{BCE}}
+ \lambda_E(t)\mathcal{L}_E
+ \lambda_S(t)\mathcal{L}_S
+ \lambda_R(t)\mathcal{L}_R
+ \lambda_C(t)\mathcal{L}_C
+ \lambda_D(t)\mathcal{L}_D.
\end{equation*}

This scheduling strategy was used consistently across all experiments unless otherwise stated.
Empirically, it significantly improved optimization stability, reduced premature rule collapse,
and enabled reliable symbolic rule extraction—particularly under probabilistic inputs and label noise.

\section{Predicate Identification}
\label{app:pred-identification}
To ensure interpretability, it is essential to explicitly extract symbolic subrules after training. Although the model learns each probability distribution \( P_{ij} \) for \( \texttt{B}_{ij} \), which reflects the importance of candidate subpredicates, a post-processing step is required to derive the final symbolic rules. 

During training, ANDRE encourages one of the probabilities in \( P_{ij} \) to converge to \( 1 \), with the remaining values approaching \( 0 \). In this case, the corresponding symbolic subpredicate is considered the most confident selection: if \( p_{ij}^{(1)} \rightarrow 1 \), \( p_{ij}^{(2)} \rightarrow 1 \), or \( p_{ij}^{(3)} \rightarrow 1 \), then \( \texttt{b}_j \), \( \neg \texttt{b}_j \), or \( \texttt{1} \) is selected, respectively.

In cases where none of the probabilities fully converges to \( 1 \), but one clearly dominates, identifying the most confident subpredicate requires evaluating the certainty of the distribution. To do so, the entropy of the distribution is computed as: $\mathcal{H} = \sum_{k=1}^{3} p_{ij}^{(k)} \log \left( p_{ij}^{(k)} + \epsilon \right)$,  where \( \epsilon \) is a small constant added for numerical stability. Lower entropy values indicate greater certainty. Based on this entropy, the most confident subpredicate \( \texttt{B}_{ij}^{\text{andre}} \) is extracted using the following criterion:
\begin{equation}
    \texttt{B}_{ij}^{\text{andre}} = \begin{cases}
        \texttt{S}_j \left[ \argmaxA_{k} \, \mathcal{S}_j [k] \right], & \text{if} \;\; \mathcal{H} \leq \eta' \\
        \texttt{1}, & \text{otherwise,}
    \end{cases}
\end{equation}
where \( \texttt{B}_{ij}^{\text{andre}} \in \texttt{B}^{n \times m} \) denotes the extracted symbolic subpredicate, and \( n \) and \( m \) are the number of subrules and body predicates, respectively. The constant \( \eta' \in [0, \ln(3)] \) is selected based on the model's final accuracy and the noise level in the data. A lower value of $\eta'$ imposes a stricter certainty requirement for subpredicate selection.

After this selection process, ANDRE constructs the matrix \( \texttt{B}_{ij}^{\text{andre}} \), where each entry corresponds to one of the symbolic subpredicates: \( \texttt{b}_j \), \( \neg \texttt{b}_j \), or \( \texttt{1} \). These entries then replace \( \texttt{B}_{ij} \) in Eq.~(\ref{subpredicates}) to form the subrules, thereby producing the final rule as specified in Eq.~(\ref{rule_form}).

\section{ANDRE's Parameters}
\label{app:andre-params}
Table~\ref{tab:andre_params} summarizes the hyperparameters used in all experiments with ANDRE's framework. These parameters were selected based on empirical tuning and prior best practices in differentiable rule learning. They govern various aspects of optimization, loss weighting, attention sharpness, and symbolic rule extraction behavior. Where applicable, different values are recommended under noisy conditions to improve robustness and generalization.
\begin{table}[H]
    \centering
    \scriptsize
    \caption{Hyperparameter Settings Used in ANDRE}
    \vspace{0.1in}
    \renewcommand{\arraystretch}{1.2}
    \begin{tabular}{lll}
        \hline
        \textbf{Parameter} & \textbf{Symbol} & \textbf{Value} \\
        \hline
        Learning rate & $\alpha$ & $0.01$ \\
        Learning rate decay & $\alpha'$ & $0.0001$ \\
        Sigmoid center & $\gamma$ & $0.5$ \\
        % Range-restricted loss center & $\gamma_r$ & $0.5$ \\
        % Connected loss center & $\gamma_c$ & $1.0$ \\
        Range-Restricted Loss Coefficient & $\eta$ & $0.1$ \\
        Connected loss coefficient & $c_1$ & $1.0$ \\
        Connected loss coefficient & $c_2$ & $12.5$ \\
        Attention sharpness coefficient & $\beta$ & $20$ \\
        Sigmoid sharpening factor & $\lambda$ & $10$ \\
        % Maximum number of subrules & $R_{\max}$ & $5$ \\
        Number of random restarts per subrule & $N_r$ & $3$ \\
        Training epochs per restart & $T$ & $500$ \\
        Batch size & $b$ & $[128, 512, 4096]$ \\
        Similarity loss coefficient & $\lambda_{\text{S}}$ & $0.2$ \\
        Entropy loss coefficient & $\lambda_{\text{E}}$ & $0.1$ \\
        Range-restricted loss coefficient & $\lambda_{\text{R}}$ & $1.00$ \\
        Connected loss coefficient & $\lambda_{\text{C}}$ & $5.00$ \\
        Digitization loss coefficient & $\lambda_{\text{D}}$ & $0.001$ \\
        Accuracy threshold for early stopping & $\tau$ & $0.95$ (reduced in noisy settings) \\
        Entropy threshold for predicate certainty & $\eta'$ & $0.4$ (standard), $0.8$ (noisy data) \\
        
        \hline
    \end{tabular}
    \label{tab:andre_params}
\end{table}

\section{Visual Examples of ANDRE}
\label{ap:visual-example}

\subsection{Structured Rule: Grandparent Task}
\addcontentsline{toc}{subsection}{Visual Example — Grandparent Task}

To further illustrate the internal learning dynamics and interpretability of ANDRE on structured relational data,
we present a visual case study on the \emph{Grandparent} task.
This task involves learning a first-order rule with auxiliary variables and multiple valid subrule instantiations,
making it a representative benchmark for evaluating both semantic and syntactic learning behavior.

\paragraph{Task Description.}
The objective is to learn the head predicate
\(\texttt{grandparent}(X_1, X_3),\)
given background knowledge consisting of \texttt{mother} and \texttt{father} relations.
The correct logical definition requires identifying transitive parent relationships through an auxiliary variable $X_2$:
\[
\texttt{grandparent}(X_1, X_3) \; \leftarrow \; \texttt{parent}(X_1, X_2) \wedge \texttt{parent}(X_2, X_3),
\]
where \texttt{parent} may be instantiated as either \texttt{mother} or \texttt{father}.

Instead of the propositionalization strategy, the trainable samples are generated using random probabilistic predicate valuations and the known target logical rules, resulting in a supervised dataset $\mathcal{E} \in \{0,1\}^{N \times (m+1)},$
where each input example consists of valuations of candidate body predicates
\[
\mathcal{B} = \{\texttt{mother}(X_1,X_2), \texttt{mother}(X_2,X_3), \texttt{father}(X_1,X_2), \texttt{father}(X_2,X_3), \dots \},
\]
and a Boolean label indicating whether the corresponding \texttt{grandparent} fact is true. In this example, we generate $1000$ samples, and $m=6$ and $N=4$ to better visualize the results.
ANDRE is trained end-to-end without prior knowledge of rule templates,
explicit clause enumeration, or predicate relevance.

\paragraph{Optimization Dynamics.}
During training, ANDRE jointly optimizes semantic and syntactic objectives using the loss formulation in Eq.~(\ref{eq:total-loss}).
The binary cross-entropy loss guides semantic correctness,
while entropy, range-restricted, connectedness, digitization, and similarity losses
gradually enforce rule sparsity, variable grounding, and structural validity.
Loss coefficients are scheduled over training as described in Appendix~\ref{app:loss-scheduling}
to stabilize optimization and avoid premature structural constraints.

Figure~\ref{fig:grandparent-training} shows the evolution of the total loss during training.
After an initial phase dominated by semantic alignment,
the loss decreases steadily as the model converges to a stable and interpretable rule structure.
% \begin{figure}[t]
%     \centering
%     \begin{subfigure}{\linewidth}
%         \centering
%         \includegraphics[width=0.48\linewidth]{training_validation_loss.pdf}
%         % \fbox{\parbox[c][4cm][c]{\linewidth}{\centering
%         % Training \& Validation Loss}}
%         % \caption{Training and validation loss over epochs.}
%         \label{fig:grandparent-loss}
%     \end{subfigure}
%     \hfill
%     \begin{subfigure}{\linewidth}
%         \centering
%         \includegraphics[width=0.48\linewidth]{training_validation_accuracy.pdf}
%         % \fbox{\parbox[c][4cm][c]{\linewidth}{\centering
%         % Training \& Validation Accuracy}}
%         % \caption{Training and validation accuracy over epochs.}
%         \label{fig:grandparent-accuracy}
%     \end{subfigure}
%     \caption{Optimization behavior of ANDRE on the Grandparent task.
%     Left: evolution of training and validation loss.
%     Right: corresponding training and validation accuracy.
%     The curves indicate stable convergence and strong generalization.}
%     \label{fig:grandparent-training}
% \end{figure}

\begin{figure}[h]
    \centering
    \subfloat[Loss]{\includegraphics[width=0.48\textwidth]{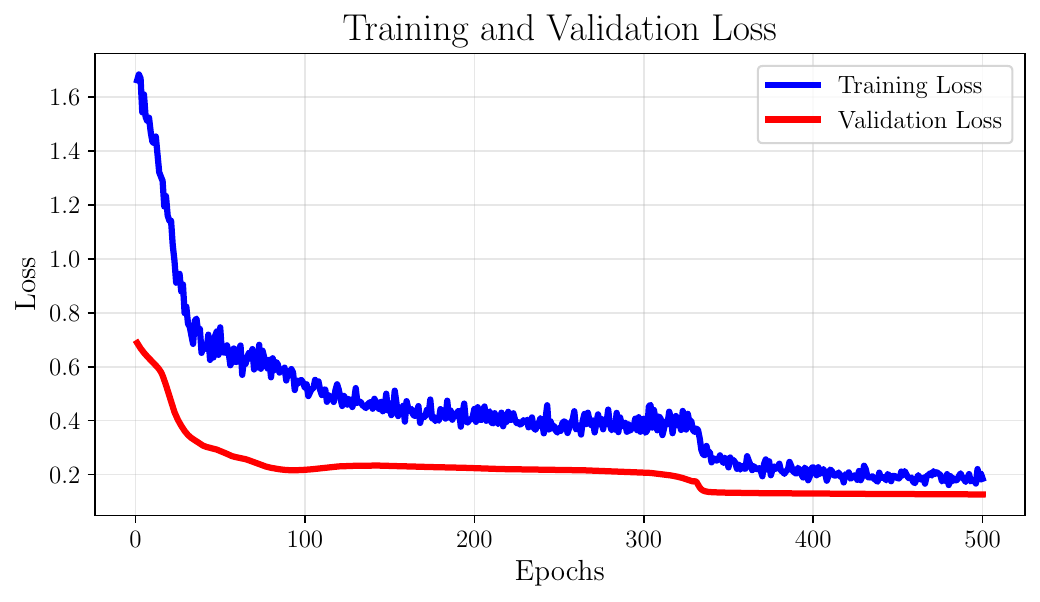}}
    \hfill
    \subfloat[Accuracy]{\includegraphics[width=0.48\textwidth]{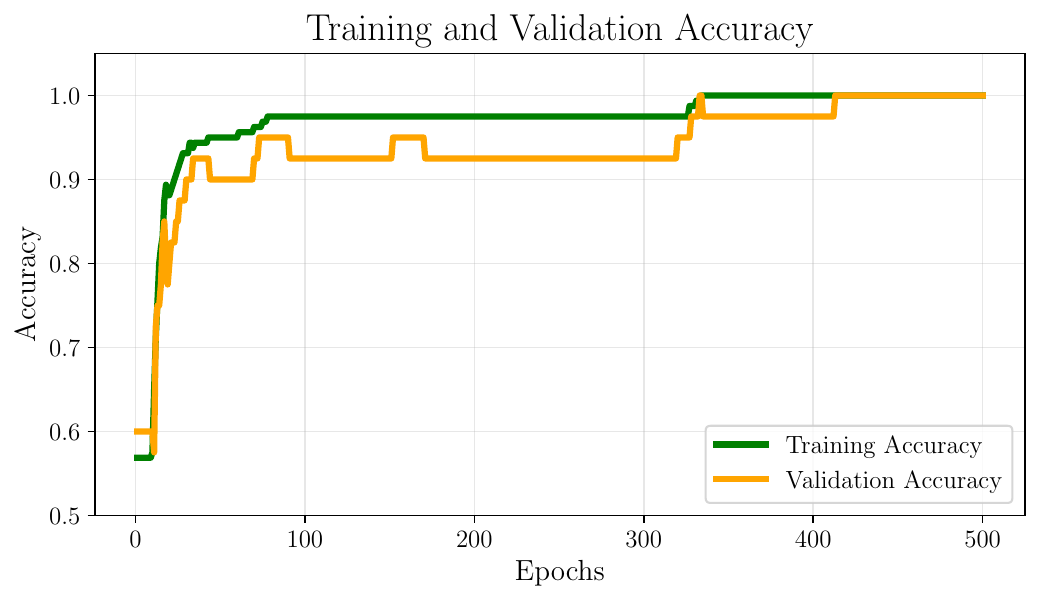}}
    \caption{Optimization behavior of ANDRE on the Grandparent task.
    Left: evolution of training and validation loss.
    Right: corresponding training and validation accuracy.
    The curves indicate stable convergence and strong generalization.}
    \label{fig:grandparent-training}
\end{figure}

\paragraph{Learned Predicate Weights and Rule Extraction.}
Upon convergence, ANDRE produces sharp probability distributions over the symbolic subpredicate space
$\{b_j, \neg b_j, 1\}$ for each body predicate in each subrule.
High-confidence selections correspond to predicates that consistently participate
in valid grandparent relationships, while irrelevant predicates are suppressed.

Figure~\ref{fig:grandparent-weights} illustrates the final learned subpredicate weights. Accordingly, after applying the predicate identification procedure described in Appendix~\ref{app:pred-identification}, ANDRE extracts the following symbolic rule set for the \texttt{grandparent} task:
\begin{equation*}
\begin{aligned}
\texttt{grandparent}(X_1, X_3) \leftarrow {} &
\texttt{mother}(X_1, X_2) \wedge \texttt{mother}(X_2, X_3), \\
\texttt{grandparent}(X_1, X_3) \leftarrow {} &
\texttt{father}(X_1, X_2) \wedge \texttt{father}(X_2, X_3), \\
\texttt{grandparent}(X_1, X_3) \leftarrow {} &
\texttt{mother}(X_1, X_2) \wedge \texttt{father}(X_2, X_3), \\
\texttt{grandparent}(X_1, X_3) \leftarrow {} &
\texttt{father}(X_1, X_2) \wedge \texttt{mother}(X_2, X_3),
\end{aligned}
\end{equation*}
which are both semantically and syntactically correct. Variables $X_1, X_3$ are head variables, while $X_2$ is an auxiliary variable.

For each rule, we present the accuracy over training and validation sets. Also, we consider ($N_b, N_r, \frac{N_r}{N_b})$ as an evaluation metric showing the quality of each rule. $N_b$ is the number of trainable samples each rule body satisfies, while $N_r$ is the number of trainable samples that both rule body and head satisfy. When $\frac{N_r}{N_b}$ is higher, it means the extracted rule is satisfying all relevant samples in the training samples. Below are the extracted rules for the grandparent task:
\begin{lstlisting}[style=stateSnapshot, caption={Extracted Rules for \texttt{grandparent}$(X_1, X_3)$ (Synthetic Dataset)}]
[1] grandparent(X1, X3) :- father(X1, X2) and mother(X2, X3).
  Training Acc: 0.6825 | Eval Acc: 0.6900
  Val Coverage: (N_b=53, N_r=53, N_r/N_b=1.00)
  Train Coverage: (N_b=219, N_r=219, N_r/N_b=1.00)

[2] grandparent(X1, X3) :- father(X1, X2) and father(X2, X3).
  Training Acc: 0.6725 | Eval Acc: 0.6850
  Val Coverage: (N_b=52, N_r=52, N_r/N_b=1.00)
  Train Coverage: (N_b=211, N_r=211, N_r/N_b=1.00)

[3] grandparent(X1, X3) :- mother(X1, X2) and mother(X2, X3).
  Training Acc: 0.6850 | Eval Acc: 0.6500
  Val Coverage: (N_b=49, N_r=47, N_r/N_b=0.9592)
  Train Coverage: (N_b=225, N_r=223, N_r/N_b=0.9911)

[4] grandparent(X1, X3) :- father(X2, X3) and mother(X1, X2).
  Training Acc: 0.6538 | Eval Acc: 0.65
  Val Coverage: (N_b=45, N_r=45, N_r/N_b=1.00)
  Train Coverage: (N_b=196, N_r=196, N_r/N_b=1.00)
\end{lstlisting}
\vspace{4mm}

% \[
% \begin{aligned}
% \texttt{grandparent}(X_1, X_3) \leftarrow {} &
% \texttt{father}(X_1, X_2) \wedge \texttt{father}(X_2, X_3). \\
% \texttt{grandparent}(X_1, X_3) \leftarrow {} &
% \texttt{mother}(X_1, X_2) \wedge \texttt{father}(X_2, X_3). \\
% \texttt{grandparent}(X_1, X_3) \leftarrow {} &
% \texttt{mother}(X_1, X_2) \wedge \texttt{mother}(X_2, X_3). \\
% \texttt{grandparent}(X_1, X_3) \leftarrow {} &
% \texttt{father}(X_1, X_2) \wedge \texttt{mother}(X_2, X_3).
% \end{aligned}
% \]

These subrules collectively form a complete and interpretable definition of the grandparent relation,
fully consistent with first-order logic and human intuition.

% \begin{figure}[h]
%     \centering
%     \includegraphics[width=0.95\linewidth]{andre_weights_visualization.pdf}
%     % \fbox{\parbox[c][5cm][c]{0.95\linewidth}{\centering Final learned ANDRE subpredicate weights}}
%     \caption{Final softmax-normalized subpredicate probabilities for the Grandparent task.
%     Each subrule converges to a sparse and interpretable predicate configuration.}
%     \label{fig:grandparent-weights}
% \end{figure}

\paragraph{Discussion.}
This visual example demonstrates ANDRE’s ability to learn multi-clause first-order rules
directly from data using a fully differentiable optimization process.
The model correctly handles auxiliary variables through syntactic regularization
and extracts multiple valid subrules without relying on predefined templates.
These results further highlight ANDRE’s robustness, interpretability,
and suitability for structured relational reasoning.

\subsection{Unstructured Rule: Arbitrary Task}
To demonstrate ANDRE's internal behavior and learning dynamics, we constructed a controlled synthetic dataset \( E \in \mathbb{R}^{100 \times (4+1)} \) consisting of 100 examples. Each row \( e \in E \) represents a training sample composed of fuzzy truth values assigned to four potential body predicates \( \texttt{b} = \{\texttt{b}_1, \texttt{b}_2, \texttt{b}_3, \texttt{b}_4\} \), and a corresponding Boolean label \( h \in \{0, 1\} \) representing the ground truth. These labels are generated based on a known logical rule composed of two subrules: $\texttt{h} \leftarrow \texttt{h}_1 \vee \texttt{h}_2$, where the subrules are defined as: $\texttt{h}_1 \leftarrow \neg \texttt{b}_2 \wedge \texttt{b}_4$ and $\texttt{h}_2 \leftarrow \texttt{b}_1 \wedge \neg \texttt{b}_3$. The corresponding fuzzy condition for generating Boolean labels is encoded as:
\begin{equation*}
h =\; \left( (1 - b_2) > 0.5 \;\land\; b_4 > 0.5 \right)
    \lor\; \left( b_1 > 0.5 \;\land\; (1 - b_3) > 0.5 \right).
\end{equation*}
The objective for ANDRE is to infer both the correct number of subrules and the structure of each subrule (i.e., which predicates to include, exclude, or negate), using only the probabilistic inputs and Boolean outputs without prior knowledge of the underlying rule structure.

The symbolic subrule matrix \( \texttt{B}_{ij}^{\text{andre}} \in \{\texttt{b}_j, \neg \texttt{b}_j, \texttt{1} \}^{n \times 4} \) is initialized as empty and must be inferred by the model during training, where \( n \) is the number of subrules, and \( j = 1,\dots,4 \) refers to the candidate predicates.

Training proceeds by feeding probabilistic valuations \( b \) into the network, computing the predicted labels \( \hat{h} \), and comparing them with the ground-truth labels \( h \). The binary cross-entropy loss \( \mathcal{L}_{\text{BCE}} \) and the entropy regularization loss \( \mathcal{L}_{\text{E}} \) are used to guide training (see Eq.~\ref{eq:bce-loss} and Eq.~\ref{eq:entropy-loss}). Since this example uses only zero-arity predicates (i.e., propositional inputs), no head or auxiliary variables are present, so all syntactic losses are deactivated in this example. Only semantic losses are active. The total loss is minimized via gradient-based backpropagation.
\begin{figure}[h]
    \centering
    \includegraphics[width=0.6\linewidth]{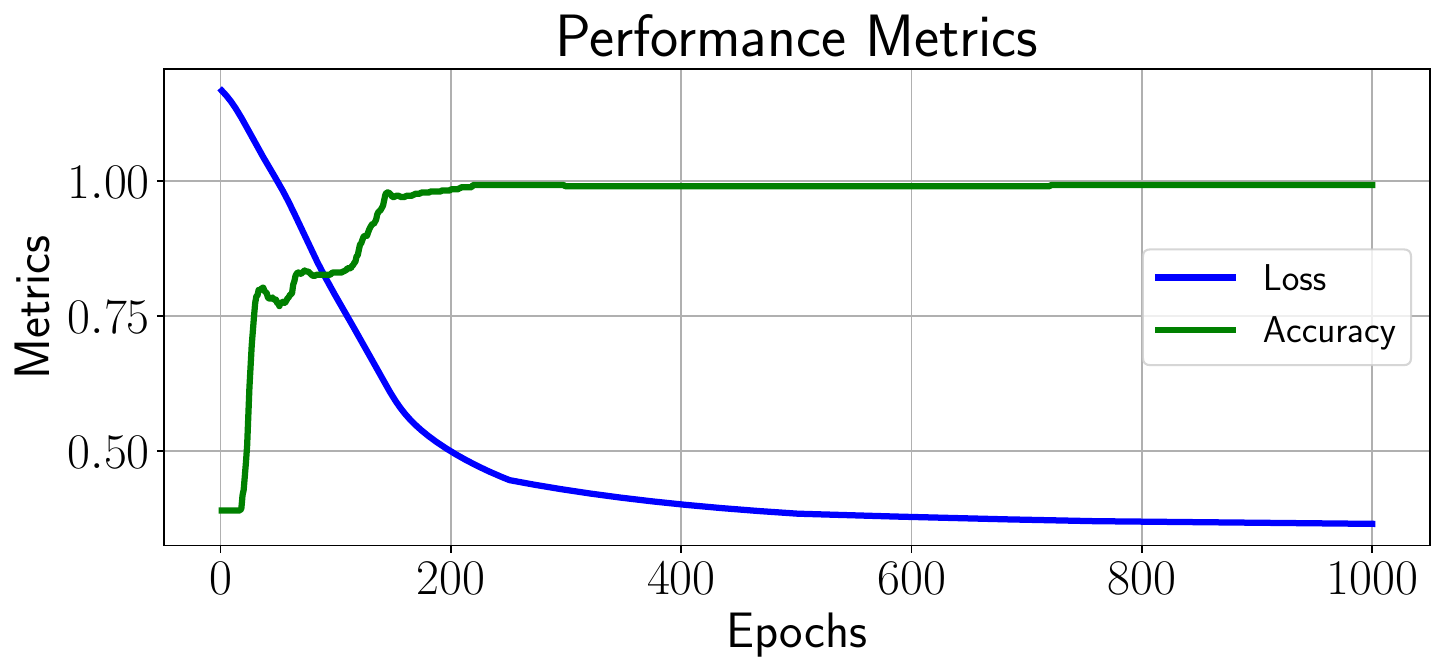}
    \caption{ANDRE performance metrics during training.}
    \label{fig:loss_acc}   
\end{figure}

Figure~\ref{fig:loss_acc} illustrates how the loss and accuracy metrics evolve during training. %The training is divided into two curriculum phases, each lasting 1000 epochs. During the first phase, only a single subrule is trained. The model achieves an accuracy of $0.822$, which does not meet the early stopping threshold \( \tau = 0.95 \). Consequently, a second subrule is introduced. In the second phase, the additional subrule allows the model to capture more complex patterns, resulting in improved performance. The accuracy increases to $0.953$ and the total loss decreases. Since the threshold is now satisfied, training terminates automatically.
Accordingly, the loss converges to the minimum value during the training, and the training accuracy converged to $1.0$, meaning that model could successfully find the global optimum.

Figure~\ref{fig:subrule_weights} depicts the evolution of the softmax-normalized weights across training. For subrule \( \hat{\texttt{h}}_1 \), the converged subpredicates are: $\{\texttt{1}, \neg \texttt{b}_2, \texttt{1}, \texttt{b}_4\}$, 
corresponding to: $\hat{\texttt{h}}_1 \leftarrow \neg \texttt{b}_2 \wedge \texttt{b}_4$. For subrule \( \hat{\texttt{h}}_2 \), the model converges to: $\{\texttt{b}_1, \texttt{1}, \neg \texttt{b}_3, \texttt{1}\}$, which translates to: $\hat{\texttt{h}}_2 \leftarrow \texttt{b}_1 \wedge \neg \texttt{b}_3$. Thus, the final symbolic matrix extracted by ANDRE is:
\begin{equation*}
\texttt{B}_{ij}^{\text{andre}} = \begin{bmatrix}
        \texttt{1}  & \neg \texttt{b}_2 & \texttt{1} & \texttt{b}_4 \\
        \texttt{b}_1 & \texttt{1} & \neg \texttt{b}_3 & \texttt{1} 
    \end{bmatrix}.
\end{equation*}

From this matrix, the overall rule reconstructed by ANDRE is:
\begin{equation*}
\hat{\texttt{h}} \leftarrow (\neg \texttt{b}_2 \wedge \texttt{b}_4) \vee (\texttt{b}_1 \wedge \neg \texttt{b}_3),
\end{equation*}
which exactly matches the target rule used to generate the dataset. The total training time was approximately $1.83$ seconds, highlighting the efficiency of ANDRE even on small datasets.

\begin{figure}[H]
    \centering
    \includegraphics[width=\linewidth]{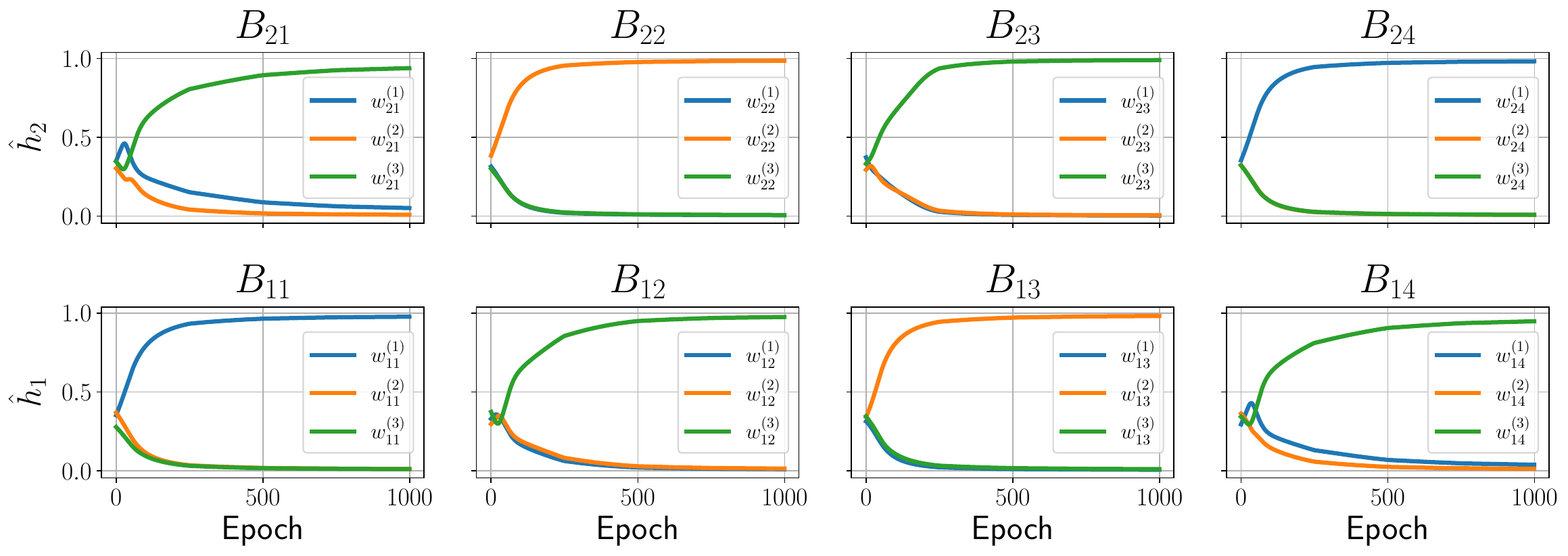}
    \caption{Final subpredicate probabilities for each \( \texttt{B}_{ij} \). Each subplot shows the softmax-normalized probability distribution over three symbolic forms: \( \texttt{b}_j \) (blue), \( \neg \texttt{b}_j \) (orange), and \( \texttt{1} \) (green). Sharp convergence indicates high confidence in subpredicate selection.}
    \label{fig:subrule_weights}
\end{figure}

\section{ANDRE Results on Classical ILP Tasks}
\label{app:classical-ilp}

This section presents ANDRE's performance on a set of classical ILP benchmark tasks widely used in the literature. For each task, we provide the formal representation of the data and background structure required for propositionalization, followed by the final symbolic rule extracted by ANDRE after training.

Throughout this section, the following notations are used consistently:
\begin{enumerate}
    \item \( X \): the complete set of logical variables used in the task.
    \item \( X^h \): the subset of variables appearing in the head predicate.
    \item \( \mathbb{E} \): the domain of possible constant values for the variables in \( X \).
    \item \( \texttt{b} = \{\texttt{b}_1, \texttt{b}_2, \ldots, \texttt{b}_m\} \): the set of candidate body predicates.
    \item \( \texttt{h} \): the symbolic head predicate to be learned.
    \item \( \mathbb{B} \): the background knowledge containing supporting facts.
    \item \( \mathcal{P} \): the set of positive examples (grounded instances of the head predicate).
    \item \( \mathcal{N} \): the set of negative examples.
\end{enumerate}

For each task, we apply a propositionalization procedure to convert the symbolic representations into a trainable dataset. After training ANDRE, we report the discovered symbolic rule, which is expected to generalize over both training and test samples.

\vspace{0.5em}
\noindent\textbf{1. The \texttt{Predecessor} Task:}

\noindent\textit{Objective:} Learn a rule that identifies when one number is the predecessor of another, using background knowledge about numeric succession.

\begin{itemize}
    \item Variable set: \( X = \{X_1, X_2\} \)
    \item Head variable set: \( X^h = X \)
    \item Domain of constants: \( \mathbb{E} = \{0, 1, \ldots, 8\} \)
    \item Body predicates: \( \texttt{b} = \{\texttt{successor}(X_2, X_1)\} \)
    \item Head predicate: \( \texttt{h} = \texttt{predecessor}(X_1, X_2) \)
    \item Background knowledge:
    \[
    \mathbb{B} = \{\texttt{successor}(X_i, X_i{+}1) \mid X_i \in \mathbb{E}\} \cup \{\texttt{zero}(0)\}
    \]
    \item Positive examples:
    \[
    \mathcal{P} = \{\texttt{predecessor}(X_j, X_i) \mid X_j = X_i{+}1, \; X_i \in \mathbb{E} \}
    \]
    \item Negative examples:
    \[
    \mathcal{N} = \{\texttt{predecessor}(X_j, X_i) \mid X_j \neq X_i{+}1, \; X_i \in \mathbb{E} \}
    \]
\end{itemize}

\noindent After training, the learned rule extracted by ANDRE is:
\begin{equation}
\texttt{predecessor}(X_1, X_2) \; \texttt{:- successor}(X_2, X_1).
\end{equation}

\noindent This result aligns perfectly with the expected inverse of the \texttt{successor} relation, demonstrating ANDRE’s ability to infer correct relational structure from structured examples and background knowledge.

\vspace{1em}
\noindent\textbf{2. The \texttt{Odd} Task:}

\noindent\textit{Objective:} Learn the logical patterns that define odd numbers using predecessor and parity relationships, and generalize them to unseen numerical values.

\begin{itemize}
    \item Variable set: \( X = \{X_1, X_2, X_3\} \)
    \item Head variable set: \( X^h = \{X_3\} \)
    \item Auxiliary variable set: \( X^a = X \setminus X^h \)
    \item Domain of constants: \( \mathbb{E} = \{0, 1, \ldots, 30\} \)
    \item Body predicates:
    \[
    \begin{aligned}
    \texttt{b} = \{&\texttt{zero}(X_1), \texttt{zero}(X_2), \texttt{zero}(X_3), \texttt{successor}(X_1,X_2), \\
    &\texttt{successor}(X_2,X_3), \texttt{successor}(X_1,X_3), \texttt{odd}(X_1), \texttt{odd}(X_2)\}
    \end{aligned}
    \]
    \item Head predicate: \( \texttt{h} = \texttt{odd}(X_3) \)
    \item Background knowledge:
    \[
    \mathbb{B} = \{\texttt{successor}(X_i, X_i{+}1) \mid X_i \in \mathbb{E}\} \cup \{\texttt{zero}(0)\}
    \]
    \item Positive examples:
    \[
    \mathcal{P} = \{\texttt{odd}(X_i) \mid X_i \bmod 2 = 1,\; X_i \in \mathbb{E} \cup \{31\}\}
    \]
    \item Negative examples:
    \[
    \mathcal{N} = \{\texttt{odd}(X_i) \mid X_i \bmod 2 = 0,\; X_i \in \mathbb{E}\}
    \]
\end{itemize}

\noindent After applying the propositionalization method, the training dataset \( E \) is constructed and ANDRE is trained. On unseen test data, the extracted subrules are:
\begin{enumerate}
    \item \texttt{odd}(X\textsubscript{3}) :- \texttt{zero}(X\textsubscript{2}), \texttt{successor}(X\textsubscript{2}, X\textsubscript{3}).
    \item \texttt{odd}(X\textsubscript{3}) :- \texttt{successor}(X\textsubscript{2}, X\textsubscript{3}), \( \neg \texttt{odd}(X\textsubscript{2}) \).
    \item \texttt{odd}(X\textsubscript{3}) :- \texttt{successor}(X\textsubscript{2}, X\textsubscript{3}), \texttt{successor}(X\textsubscript{1}, X\textsubscript{2}), \texttt{odd}(X\textsubscript{1}).
\end{enumerate}

\vspace{1em}
\noindent\textbf{3. The \texttt{Even} Task:}

\noindent\textit{Objective:} Discover the rule structure that governs even numbers, exploiting arithmetic successor relationships and parity predicates.

\begin{itemize}
    \item Variable set: \( X = \{X_1, X_2, X_3\} \)
    \item Head variable set: \( X^h = \{X_3\} \)
    \item Auxiliary variable set: \( X^a = X \setminus X^h \)
    \item Domain of constants: \( \mathbb{E} = \{0, 1, \ldots, 30\} \)
    \item Body predicates:
    \[
    \begin{aligned}
    \texttt{b} = \{&\texttt{zero}(X_1), \texttt{zero}(X_2), \texttt{zero}(X_3), \texttt{successor}(X_1,X_2), \\
    &\texttt{successor}(X_2,X_3), \texttt{successor}(X_1,X_3), \texttt{even}(X_1), \texttt{even}(X_2)\}
    \end{aligned}
    \]
    \item Head predicate: \( \texttt{h} = \texttt{even}(X_3) \)
    \item Background knowledge:
    \[
    \mathbb{B} = \{\texttt{successor}(X_i, X_i{+}1) \mid X_i \in \mathbb{E}\} \cup \{\texttt{zero}(0)\}
    \]
    \item Positive examples:
    \[
    \mathcal{P} = \{\texttt{even}(X_i) \mid X_i \bmod 2 = 0,\; X_i \in \mathbb{E}\}
    \]
    \item Negative examples:
    \[
    \mathcal{N} = \{\texttt{even}(X_i) \mid X_i \bmod 2 = 1,\; X_i \in \mathbb{E}\}
    \]
\end{itemize}

\noindent After training, ANDRE extracted the following subrules from the propositionalized dataset:
\begin{enumerate}
    \item \texttt{even}(X\textsubscript{3}) :- \texttt{zero}(X\textsubscript{3}).
    \item \texttt{even}(X\textsubscript{3}) :- \texttt{successor}(X\textsubscript{2}, X\textsubscript{3}), \( \neg \texttt{even}(X\textsubscript{2}) \).
    \item \texttt{even}(X\textsubscript{3}) :- \texttt{successor}(X\textsubscript{2}, X\textsubscript{3}), \texttt{successor}(X\textsubscript{1}, X\textsubscript{2}), \texttt{even}(X\textsubscript{1}).
\end{enumerate}

\vspace{1em}
\noindent\textbf{4. The \texttt{LessThan} Task:}

\noindent\textit{Objective:} Learn transitive and arithmetic rules that define the less-than relation between two integers using a successor-based formulation.

\begin{itemize}
    \item Variable set: \( X = \{X_1, X_2, X_3\} \)
    \item Head variable set: \( X^h = \{X_1, X_3\} \)
    \item Auxiliary variable set: \( X^a = \{X_2\} \)
    \item Domain of constants: \( \mathbb{E} = \{0, 1, \ldots, 9\} \)
    \item Body predicates:
    \[
    \begin{aligned}
    \texttt{b} = \{ &\texttt{successor}(X_1,X_2), \texttt{successor}(X_2,X_3), \texttt{successor}(X_1,X_3),\\
    &\texttt{lessThan}(X_3,X_2), \texttt{lessThan}(X_2,X_1)\}
    \end{aligned}
    \]
    \item Head predicate: \( \texttt{h} = \texttt{lessThan}(X_3, X_1) \)
    \item Background knowledge:
    \[
    \mathbb{B} = \{\texttt{successor}(X_i, X_i{+}1) \mid X_i \in \mathbb{E}\} \cup \{\texttt{zero}(0)\}
    \]
    \item Positive examples:
    \[
    \mathcal{P} = \{\texttt{lessThan}(X_i, X_j) \mid X_i < X_j,\; X_i, X_j \in \mathbb{E}\}
    \]
    \item Negative examples:
    \[
    \mathcal{N} = \{\texttt{lessThan}(X_i, X_j) \mid X_i \geq X_j,\; X_i, X_j \in \mathbb{E}\}
    \]
\end{itemize}

\noindent After applying the propositionalization strategy, ANDRE is trained to learn rules. The test set includes integer pairs beyond 9 to evaluate generalization. The extracted subrules are:
\begin{enumerate}
    \item \texttt{lessThan}(X\textsubscript{3}, X\textsubscript{1}) :- \texttt{successor}(X\textsubscript{1}, X\textsubscript{3}).
    \item \texttt{lessThan}(X\textsubscript{3}, X\textsubscript{1}) :- \texttt{lessThan}(X\textsubscript{3}, X\textsubscript{2}), \texttt{lessThan}(X\textsubscript{2}, X\textsubscript{1}).
\end{enumerate}

\vspace{1em}
\noindent\textbf{5. The \texttt{Grandparent} Task:}

\noindent\textit{Objective:} Learn the definition of a grandparent based on transitive parent relationships using both \texttt{mother} and \texttt{father} facts provided in the background knowledge.

\begin{itemize}
    \item Variable set: \( X = \{X_1, X_2, X_3\} \)
    \item Head variable set: \( X^h = \{X_1, X_3\} \)
    \item Auxiliary variable set: \( X^a = \{X_2\} \)
    \item Body predicates:
    \[
    \begin{aligned}
    \texttt{b} = \{ &\texttt{father}(X_1,X_2), \texttt{father}(X_2,X_3), \texttt{father}(X_1,X_3),\\
    &\texttt{mother}(X_1,X_2), \texttt{mother}(X_2,X_3), \texttt{mother}(X_1,X_3)\}
    \end{aligned}
    \]
    \item Head predicate: \( \texttt{h} = \texttt{grandparent}(X_1, X_3) \)
    \item Background knowledge:
    \[
    \begin{aligned}
    \mathbb{B} = \{ &\texttt{mother}(a,c),\; \texttt{mother}(c,e),\; \texttt{mother}(b,d),\; \texttt{mother}(d,f), \\
    &\texttt{father}(g,h),\; \texttt{father}(h,i),\; \texttt{father}(j,k),\; \texttt{father}(k,l), \\
    &\texttt{mother}(m,n),\; \texttt{father}(n,o),\; \dots \}
    \end{aligned}
    \]
    \item Positive examples:
    \[
    \begin{aligned}
    \mathcal{P} = \{&\texttt{grandparent}(a,e),\; \texttt{grandparent}(b,f),\; \texttt{grandparent}(g,i), \\
    &\texttt{grandparent}(j,l),\; \texttt{grandparent}(m,o),\; \dots \}
    \end{aligned}
    \]
\end{itemize}

\noindent Using the extracted dataset, ANDRE learns the following subrules:
\begin{enumerate}
    \item \texttt{grandparent}(X\textsubscript{1}, X\textsubscript{3}) :- \texttt{mother}(X\textsubscript{1}, X\textsubscript{2}), \texttt{mother}(X\textsubscript{2}, X\textsubscript{3}).
    \item \texttt{grandparent}(X\textsubscript{1}, X\textsubscript{3}) :- \texttt{father}(X\textsubscript{1}, X\textsubscript{2}), \texttt{father}(X\textsubscript{2}, X\textsubscript{3}).
    \item \texttt{grandparent}(X\textsubscript{1}, X\textsubscript{3}) :- \texttt{mother}(X\textsubscript{1}, X\textsubscript{2}), \texttt{father}(X\textsubscript{2}, X\textsubscript{3}).
    \item \texttt{grandparent}(X\textsubscript{1}, X\textsubscript{3}) :- \texttt{father}(X\textsubscript{1}, X\textsubscript{2}), \texttt{mother}(X\textsubscript{2}, X\textsubscript{3}).
\end{enumerate}

\vspace{1em}
\noindent\textbf{6. The \texttt{Son} Task:}

\noindent\textit{Objective:} Learn how the \texttt{son} relationship can be derived from \texttt{father}, \texttt{brother}, and \texttt{sister} facts using transitivity and kinship inference.

\begin{itemize}
    \item Variable set: \( X = \{X_1, X_2, X_3\} \)
    \item Head variable set: \( X^h = \{X_1, X_3\} \)
    \item Auxiliary variable set: \( X^a = \{X_2\} \)
    \item Body predicates:
    \[
    \begin{aligned}
    \texttt{b} = \{ & \texttt{father}(X_1,X_2),\; \texttt{father}(X_2,X_3),\; \texttt{father}(X_1,X_3), \\
    & \texttt{brother}(X_1,X_2),\; \texttt{sister}(X_1,X_2),\; \texttt{son}(X_2,X_3) \}
    \end{aligned}
    \]
    \item Head predicate: \( \texttt{h} = \texttt{son}(X_1, X_3) \)
    \item Background knowledge:
    \[
    \begin{aligned}
    \mathbb{B} = \{ &\texttt{father}(a,b),\; \texttt{father}(a,c),\; \texttt{father}(d,e),\; \texttt{father}(d,f), \\
    &\texttt{father}(g,h),\; \texttt{father}(g,i),\; \texttt{brother}(b,c),\; \texttt{brother}(c,b), \\
    &\texttt{brother}(e,f),\; \texttt{sister}(f,e),\; \texttt{sister}(h,i),\; \texttt{sister}(i,h),\; \dots \}
    \end{aligned}
    \]
    \item Positive examples:
    \[
    \mathcal{P} = \{\texttt{son}(b,a),\; \texttt{son}(c,a),\; \texttt{son}(e,d),\; \dots \}
    \]
\end{itemize}

\noindent After applying the propositionalization strategy, ANDRE is trained and infers the following subrules:
\begin{enumerate}
    \item \texttt{son}(X\textsubscript{1}, X\textsubscript{3}) :- \texttt{brother}(X\textsubscript{1}, X\textsubscript{2}), \texttt{son}(X\textsubscript{2}, X\textsubscript{3}).
    \item \texttt{son}(X\textsubscript{1}, X\textsubscript{3}) :- \texttt{father}(X\textsubscript{3}, X\textsubscript{1}).
\end{enumerate}

\vspace{1em}
\noindent\textbf{7. The \texttt{Relatedness} Task:}

\noindent\textit{Objective:} Determine whether two individuals are related, based on transitive closure over \texttt{parent} relationships and recursive definitions of \texttt{related}.

\begin{itemize}
    \item Variable set: \( X = \{X_1, X_2, X_3\} \)
    \item Head variable set: \( X^h = \{X_1, X_3\} \)
    \item Auxiliary variable set: \( X^a = \{X_2\} \)
    \item Body predicates:
    \[
    \begin{aligned}
    \texttt{b} = \{ & \texttt{parent}(X_1,X_2),\; \texttt{parent}(X_2,X_3),\; \texttt{parent}(X_1,X_3), \\
    & \texttt{related}(X_1,X_2),\; \texttt{related}(X_2,X_3) \}
    \end{aligned}
    \]
    \item Head predicate: \( \texttt{h} = \texttt{related}(X_1, X_3) \)
    \item Background knowledge:
    \[
    \begin{aligned}
    \mathbb{B} = \{ &\texttt{parent}(a,b),\; \texttt{parent}(a,c),\; \texttt{parent}(c,e),\; \texttt{parent}(c,f), \\
    &\texttt{parent}(d,c),\; \texttt{parent}(g,h),\; \dots \}
    \end{aligned}
    \]
    \item Positive examples:
    \[
    \begin{aligned}
    \mathcal{P} = \{ &\texttt{related}(a,b),\; \texttt{related}(a,c),\; \texttt{related}(a,e),\; \texttt{related}(a,f), \\
    &\texttt{related}(f,a),\; \texttt{related}(a,a),\; \texttt{related}(d,b),\; \texttt{related}(h,g),\; \dots \}
    \end{aligned}
    \]
\end{itemize}

\noindent ANDRE is then trained using the propositionalized dataset. The extracted subrules are as follows:
\begin{enumerate}
    \item \texttt{related}(X\textsubscript{1}, X\textsubscript{3}) :- \texttt{parent}(X\textsubscript{1}, X\textsubscript{3}).
    \item \texttt{related}(X\textsubscript{1}, X\textsubscript{3}) :- \texttt{parent}(X\textsubscript{3}, X\textsubscript{1}).
    \item \texttt{related}(X\textsubscript{1}, X\textsubscript{3}) :- \texttt{related}(X\textsubscript{3}, X\textsubscript{1}).
    \item \texttt{related}(X\textsubscript{1}, X\textsubscript{3}) :- \texttt{related}(X\textsubscript{1}, X\textsubscript{2}), \texttt{related}(X\textsubscript{2}, X\textsubscript{3}).
\end{enumerate}

\vspace{1em}
\noindent\textbf{8. The \texttt{Father} Task:}

\noindent\textit{Objective:} Infer the \texttt{father} relationship using background assumptions involving marriage and motherhood.

\begin{itemize}
    \item Variable set: \( X = \{X_1, X_2, X_3\} \)
    \item Head variable set: \( X^h = \{X_1, X_3\} \)
    \item Auxiliary variable set: \( X^a = \{X_2\} \)
    \item Body predicates:
    \[
    \begin{aligned}
    \texttt{b} = \{ & \texttt{mother}(X_1,X_2),\; \texttt{mother}(X_2,X_3),\; \texttt{mother}(X_1,X_3), \\
    & \texttt{husband}(X_1,X_2),\; \texttt{husband}(X_2,X_3),\; \texttt{husband}(X_1,X_3) \}
    \end{aligned}
    \]
    \item Head predicate: \( \texttt{h} = \texttt{father}(X_1, X_3) \)
    \item Background knowledge:
    \[
    \begin{aligned}
    \mathbb{B} = \{ 
    &\texttt{aunt\_of}(\text{Howard}, \text{Anne}),\;
    \texttt{aunt\_of}(\text{Anne}, \text{Henry}), \\
    &\texttt{mother\_of}(\text{Anne}, \text{Elizabeth}),\;
    \texttt{married\_to}(\text{Henry}, \text{Anne}), \\
    &\texttt{brother\_of}(\text{John}, \text{Margaret}),\;
    \texttt{brother\_of}(\text{Henry}, \text{Louis}), \\
    &\texttt{married\_to}(\text{Louis}, \text{Adele}),\;
    \texttt{mother\_of}(\text{Adele}, \text{Philip}), \\
    &\texttt{brother\_of}(\text{Philip}, \text{Margaret}),\; \ldots 
    \}
    \end{aligned}
    \]
    \item Positive examples:
    \[
    \mathcal{P} = \{\texttt{father}(\text{Louis}, \text{Philip}),\;
                   \texttt{father}(\text{Henry}, \text{Elizabeth}),\; \ldots \}
    \]
\end{itemize}

\noindent After applying the propositionalization strategy, ANDRE is trained and infers the following subrule:
\begin{enumerate}
    \item \texttt{father}(X\textsubscript{1}, X\textsubscript{3}) :- \texttt{husband}(X\textsubscript{1}, X\textsubscript{2}), \texttt{mother}(X\textsubscript{2}, X\textsubscript{3}).
\end{enumerate}

\vspace{1em}
\noindent\textbf{9. The \texttt{Directed Edge} Task:}

\noindent\textit{Objective:} Determine whether two nodes are connected by a directed edge in either direction, using basic edge facts.

\begin{itemize}
    \item Variable set: \( X = \{X_1, X_2\} \)
    \item Head variable set: \( X^h = \{X_1, X_2\} \)
    \item Body predicates:
    \[
    \texttt{b} = \{\texttt{edge}(X_1,X_2),\; \texttt{d-edge}(X_2,X_1)\}
    \]
    \item Head predicate: \( \texttt{h} = \texttt{d-edge}(X_1, X_2) \)
    \item Background knowledge:
    \[
    \mathbb{B} = \{ \texttt{edge}(a,b),\; \texttt{edge}(b,d),\; \texttt{edge}(c,c),\; \ldots \}
    \]
    \item Positive examples:
    \[
    \mathcal{P} = \{ \texttt{d-edge}(a,b),\; \texttt{d-edge}(b,a),\; 
                    \texttt{d-edge}(b,d),\; \texttt{d-edge}(d,b),\;
                    \texttt{d-edge}(c,c),\; \ldots \}
    \]
\end{itemize}

\noindent After propositionalization, ANDRE learns the following rules:
\begin{enumerate}
    \item \texttt{d-edge}(X\textsubscript{1}, X\textsubscript{2}) :- \texttt{d-edge}(X\textsubscript{2}, X\textsubscript{1}).
    \item \texttt{d-edge}(X\textsubscript{1}, X\textsubscript{2}) :- \texttt{edge}(X\textsubscript{1}, X\textsubscript{2}).
\end{enumerate}

\vspace{1em}
\noindent\textbf{10. The \texttt{Connectedness} Task:}

\noindent\textit{Objective:} Learn the \texttt{connectedness} relation, which holds true if there exists a direct or transitive path (via one or more \texttt{edge} relations) between two nodes.

\begin{itemize}
    \item Variable set: \( X = \{X_1, X_2, X_3\} \)
    \item Head variable set: \( X^h = \{X_1, X_3\} \)
    \item Auxiliary variable set: \( X^a = \{X_2\} \)
    \item Body predicates:
    \[
    \begin{aligned}
    \texttt{b} = \{ &\texttt{edge}(X_1,X_3),\; \texttt{edge}(X_3,X_1),\; \texttt{edge}(X_2,X_3), \\
    &\texttt{connectedness}(X_1,X_2),\; \texttt{connectedness}(X_2,X_3) \}
    \end{aligned}
    \]
    \item Head predicate: \( \texttt{h} = \texttt{connectedness}(X_1, X_3) \)
    \item Background knowledge:
    \[
    \mathbb{B} = \{ \texttt{edge}(a,b),\; \texttt{edge}(b,c),\; \texttt{edge}(c,d),\; \texttt{edge}(b,a),\; \ldots \}
    \]
    \item Positive examples:
    \[
    \begin{aligned}
    \mathcal{P} = \{&
    \texttt{connectedness}(a,b),\; \texttt{connectedness}(b,c),\; \texttt{connectedness}(c,d),\; \\ &\texttt{connectedness}(b,a), \;
    \texttt{connectedness}(a,c),\; \texttt{connectedness}(a,d),\;  \\ &\texttt{connectedness}(a,a), \;\texttt{connectedness}(b,d),\; \texttt{connectedness}(b,b),\; \ldots \}
    \end{aligned}
    \]
\end{itemize}

\noindent After applying the propositionalization strategy, the dataset \( E \) is generated and used to train ANDRE. The learned symbolic rules are as follows:
\begin{enumerate}
    \item \texttt{connectedness}(X\textsubscript{1}, X\textsubscript{3}) :- \texttt{edge}(X\textsubscript{1}, X\textsubscript{3}).
    \item \texttt{connectedness}(X\textsubscript{1}, X\textsubscript{3}) :- \texttt{edge}(X\textsubscript{3}, X\textsubscript{1}).
    \item \texttt{connectedness}(X\textsubscript{1}, X\textsubscript{3}) :- \texttt{connectedness}(X\textsubscript{1}, X\textsubscript{2}), \texttt{edge}(X\textsubscript{2}, X\textsubscript{3}).
\end{enumerate}

\section{ANDRE on Knowledge Bases}
\addcontentsline{toc}{section}{Appendix X: ANDRE on Knowledge-base Datasets}
\label{app:andre-knowledge-base}
Table~\ref{tab:dataset_statistics} shows the statistical details of the knowledge bases used to evaluate the performance of ANDRE.
\begin{table}[h]
\centering
\scriptsize
\caption{Dataset statistics of Countries, Nations, UMLS, Kinship, UW-CSE, and
Alzheimers-amine.}
\label{tab:dataset_statistics}
\begin{tabular}{lccccc}
\toprule
\textbf{Dataset} & \textbf{\#Constant} & \textbf{\#Relation} & \textbf{\#Raw Facts} & \textbf{\# Trainable Samples} & \textbf{\# Test Samples} \\
\midrule
Countries-S1 & 252    & 2   & 1,110 & 1,812,378 & 95,388 \\
Countries-S2 & 252    & 2   & 1,063 & 1,197,092 & 299,272 \\
Nations            & 14     & 56  & 2,364              & 36,496 & 1,920 \\
UMLS               & 135    & 49  & 5,598              & 715,609 & 37,663 \\
% Kinship            & 104    & 26  & 9,612              & 1,074 \\
UW-CSE             & 1,209  & 15  & 2,675              & 3,861,667 & 429,074\\
Alzheimers   & 147    & 32  & 980                & 46,337 & 2,438 \\
% WN18               & 40,943 & 18  & 146,442            & 5,000 \\
% WN18RR             & 40,943 & 11  & 89,869             & 3,134 \\
% FB15KSelected      & 14,541 & 237 & 289,650            & 20,466 \\
\bottomrule
\end{tabular}
\end{table}
The rest of appendix reports some symbolic rules extracted by ANDRE on knowledge-base reasoning benchmarks.
For each task, we present representative rules discovered across different restarts and rule configurations.
All rules are expressed in first-order logic form and extracted using the predicate identification procedure
described in Section~\ref{app:pred-identification}.
Only training and validation accuracy is reported here for brevity.

\subsection{Countries Dataset}
\addcontentsline{toc}{subsection}{Countries Dataset}

The \textit{Countries} task aims to learn the target predicate
\(\texttt{locatedIn}(X_1, X_2),\)
using background relations such as \texttt{neighborOf} and recursive \texttt{locatedIn} facts.
Below, we present all representative rules extracted by ANDRE, formatted in a prompt-style box for clarity.

\begin{lstlisting}[style=stateSnapshot, caption={Extracted Rules for \texttt{locatedIn}$(X_1, X_2)$ (Nations Dataset)}]
[1] locatedIn(X1, X2) :- locatedIn(X3, X2) and neighborOf(X3, X1).
  Training Acc: 0.9927 | Eval Acc: 0.9930

[2] locatedIn(X1, X2) :- locatedIn(X1, X3) and locatedIn(X3, X2).
  Training Acc: 0.9926 | Eval Acc: 0.9930

[3] locatedIn(X1, X2) :- locatedIn(X3, X1) and neighborOf(X2, X3).
  Training Acc: 0.9925 | Eval Acc: 0.9928

[4] locatedIn(X1, X2) :- locatedIn(X3, X1) and neighborOf(X3, X2).
  Training Acc: 0.9925 | Eval Acc: 0.9928

[5] locatedIn(X1, X2) :- locatedIn(X1, X3) and neighborOf(X3, X2).
  Training Acc: 0.9925 | Eval Acc: 0.9928

[6] locatedIn(X1, X2) :- neighborOf(X1, X3) and neighborOf(X2, X3).
  Training Acc: 0.9925 | Eval Acc: 0.9928

[7] locatedIn(X1, X2) :- locatedIn(X2, X3) and neighborOf(X3, X1).
  Training Acc: 0.9925 | Eval Acc: 0.9928

\end{lstlisting}

\vspace{3mm}
% \noindent\textbf{Note.}
% Due to the strong class imbalance in the Countries dataset, multiple structurally distinct rules
% achieve near-identical accuracy despite differing semantic validity.

\subsection{Nations Dataset}
\addcontentsline{toc}{subsection}{Nations Dataset}

The \textit{Nations} task aims to learn semantic relations between geopolitical entities
using background predicates describing economic, political, and geographic interactions.
Below, we report representative rules extracted by ANDRE for different target predicates.

\subsubsection*{\texttt{blockpositionindex}$(X_1, X_2)$}
\addcontentsline{toc}{subsubsection}{\texttt{isa}$(X_1, X_2)$}

The target predicate \texttt{blockpositionindex}$(X_1, X_2)$ captures hierarchical relationships between entities
in the Nations knowledge base. The following rules were extracted by ANDRE:

\begin{lstlisting}[style=stateSnapshot, caption={Extracted Rules for \texttt{blockpositionindex}$(X_1, X_2)$ (Nations Dataset)}]
[1] blockpositionindex(X1, X2) :- blockpositionindex(X2, X1).
  Training Acc: 0.9590 | Eval Acc: 0.9620
  Val Coverage: (N_b=524, N_r=488, N_r/N_b=0.9313)
  Train Coverage: (N_b=9668, N_r=8920, N_r/N_b=0.9226)

[2] blockpositionindex(X1, X2) :-not(timesincewar(X1, X2)) and blockpositionindex(X2, X1).
  Training Acc: 0.9474 | Eval Acc: 0.9474
  Val Coverage: (N_b=486, N_r=455, N_r/N_b=0.9362)
  Train Coverage: (N_b=9019, N_r=8383, N_r/N_b=0.9295)

[3] blockpositionindex(X1, X2) :- commonbloc0(X1, X2).
  Training Acc: 0.8625 | Eval Acc: 0.8568
  Val Coverage: (N_b=290, N_r=270, N_r/N_b=0.9310)
  Train Coverage: (N_b=5394, N_r=5022, N_r/N_b=0.9310)

[4] blockpositionindex(X1, X2) :- commonbloc0(X2, X1).
  Training Acc: 0.8627 | Eval Acc: 0.8536
  Val Coverage: (N_b=284, N_r=264, N_r/N_b=0.9296)
  Train Coverage: (N_b=5400, N_r=5028, N_r/N_b=0.9311)

[5] blockpositionindex(X1, X2) :- not(relintergovorgs(X2, X1)) and embassy(X2, X1) and not(commonbloc2(X1, X2)) and not(reltreaties(X1, X2)) and not(reldiplomacy(X1, X2)) and conferences(X2, X1).
  Training Acc: 0.7538 | Eval Acc: 0.7474
  Val Coverage: (N_b=44, N_r=42, N_r/N_b=0.9545)
  Train Coverage: (N_b=803, N_r=742, N_r/N_b=0.9240)
\end{lstlisting}

\vspace{2mm}
\noindent\textbf{Note.}
% Rules for the \texttt{blockpositionindex} predicate vary significantly in predictive strength, reflecting differences in semantic alignment and coverage across background relations in the Nations knowledge base.
We refer to Appendix~\ref{ap:visual-example} for the definitions of $(N_r, N_b, \frac{N_r}{N_b})$.

\vspace{3mm}
\subsubsection*{\texttt{intergovorgs3}$(X_1, X_2)$}
\addcontentsline{toc}{subsubsection}{\texttt{intergovorgs3}$(X_1, X_2)$}

The target predicate \texttt{intergovorgs3}$(X_1, X_2)$ models intergovernmental organization
relationships in the Nations knowledge base.
Below, we present representative rules extracted by ANDRE after removing redundant variants with unnecessary or overlapping literals.

\begin{lstlisting}[style=stateSnapshot, caption={Extracted Rules for \texttt{intergovorgs}$(X_1, X_2)$ (Nations Dataset)}]
[1] intergovorgs3(X1, X2) :- embassy(X1, X2) and ngoorgs3(X1, X2) and not(ngoorgs3(X2, X1)) and not(exports3(X1, X2)) and not(releconomicaid(X2, X1)) and not(releconomicaid(X3, X1)) and not(duration(X2, X1)) and not(lostterritory(X1, X3)).
      Training Acc: 0.7466 | Eval Acc: 0.7482
      Val Coverage: (N_b=648, N_r=529, N_r/N_b=0.8164)
      Train Coverage: (N_b=12647, N_r=10135, N_r/N_b=0.8014)

[2] intergovorgs3(X1, X2) :- ngoorgs3(X1, X2) and not(economicaid(X2, X1)) and ngo(X2, X1) and not(relngo(X2, X1)) and not(relexportbooks(X1, X3)) and not(violentactions(X2, X1)) and not(warning(X1, X3)).
  Training Acc: 0.6836 | Eval Acc: 0.6819
  Val Coverage: (N_b=430, N_r=370, N_r/N_b=0.8605)
  Train Coverage: (N_b=8473, N_r=7145, N_r/N_b=0.8433)

[3] intergovorgs3(X1, X2) :- not(accusation(X1, X2)) and intergovorgs(X1, X2) and ngo(X2, X1) and not(aidenemy(X2, X3)) and not(releconomicaid(X2, X1)) and not(expeldiplomats(X2, X1)) and treaties(X2, X1) and not(lostterritory(X1, X3)).
  Training Acc: 0.6867 | Eval Acc: 0.6806
  Val Coverage: (N_b=466, N_r=387, N_r/N_b=0.8305)
  Train Coverage: (N_b=8934, N_r=7420, N_r/N_b=0.8305)

[4] intergovorgs3(X1, X2) :- not(accusation(X2, X1)) and not(commonbloc2(X1, X2)) and relngo(X1, X2) and timesinceally(X1, X2) and not(relemigrants(X2, X3)) and not(lostterritory(X2, X3)).
  Training Acc: 0.6551 | Eval Acc: 0.6733
  Val Coverage: (N_b=347, N_r=322, N_r/N_b=0.9280)
  Train Coverage: (N_b=6473, N_r=5737, N_r/N_b=0.8863)

[5] intergovorgs3(X1, X2) :- not(militaryalliance(X1, X2)) and intergovorgs(X1, X2) and not(expeldiplomats(X2, X3)) and not(relngo(X2, X1)) and not(relexportbooks(X1, X3)).
  Training Acc: 0.6724 | Eval Acc: 0.6713
  Val Coverage: (N_b=312, N_r=303, N_r/N_b=0.9712)
  Train Coverage: (N_b=5949, N_r=5723, N_r/N_b=0.9620)

[6] intergovorgs3(X1, X2) :- not(commonbloc2(X1, X2)) and ngoorgs3(X1, X2) and relngo(X1, X2) and not(relngo(X2, X1)) and not(students(X1, X3)) and not(tourism(X2, X1)) and not(dependent(X1, X3)) and not(violentactions(X1, X2)) and not(severdiplomatic(X1, X3)).
  Training Acc: 0.6517 | Eval Acc: 0.6468
  Val Coverage: (N_b=349, N_r=303, N_r/N_b=0.8682)
  Train Coverage: (N_b=6789, N_r=5846, N_r/N_b=0.8611)

[7] intergovorgs3(X1, X2) :- relintergovorgs(X1, X2) and not(economicaid(X2, X3)) and intergovorgs(X2, X1) and not(eemigrants(X3, X2)) and timesinceally(X2, X1).
  Training Acc: 0.6106 | Eval Acc: 0.6163
  Val Coverage: (N_b=261, N_r=236, N_r/N_b=0.9042)
  Train Coverage: (N_b=4515, N_r=4119, N_r/N_b=0.9123)

[8] intergovorgs3(X1, X2) :- relintergovorgs(X1, X2) and intergovorgs(X2, X1) and timesinceally(X2, X1) and not(exportbooks(X1, X2)) and not(dependent(X3, X1)) and not(warning(X1, X3)).
  Training Acc: 0.6166 | Eval Acc: 0.6130
  Val Coverage: (N_b=256, N_r=231, N_r/N_b=0.9023)
  Train Coverage: (N_b=4676, N_r=4286, N_r/N_b=0.9166)
\end{lstlisting}
\vspace{5mm}
% \noindent\textbf{Note.}
% Multiple rules achieve comparable accuracy despite differing semantic plausibility, reflecting the complexity and noise present in relational geopolitical data.

\vspace{3mm}
\subsubsection*{\texttt{negativecomm}$(X_1, X_2)$}
\addcontentsline{toc}{subsubsection}{\texttt{negativecomm}$(X_1, X_2)$}

The target predicate \texttt{negativecomm}$(X_1, X_2)$ models hostile or adversarial communication
between entities in the Nations knowledge base.
Below, we present representative rules extracted by ANDRE, after removing redundant symmetric variants and rules with overlapping recursive structure.

\begin{lstlisting}[style=stateSnapshot, caption={Extracted Rules for \texttt{negativecomm}$(X_1, X_2)$ (Nations Dataset)}]
[1] negativecomm(X1, X2) :- negativebehavior(X1, X2) and timesinceally(X2, X1).
  Training Acc: 0.9084 | Eval Acc: 0.9223
  Val Coverage: (N_b=124, N_r=124, N_r/N_b=1.0000)
  Train Coverage: (N_b=2228, N_r=2228, N_r/N_b=1.0000)

[2] negativecomm(X1, X2) :- negativebehavior(X1, X2) and accusation(X1, X2).
  Training Acc: 0.9208 | Eval Acc: 0.9142
  Val Coverage: (N_b=141, N_r=129, N_r/N_b=0.9149)
  Train Coverage: (N_b=2799, N_r=2615, N_r/N_b=0.9343)

[3] negativecomm(X1, X2) :- negativebehavior(X1, X2) and blockpositionindex(X1, X2).
  Training Acc: 0.8974 | Eval Acc: 0.9049
  Val Coverage: (N_b=137, N_r=123, N_r/N_b=0.8978)
  Train Coverage: (N_b=2411, N_r=2229, N_r/N_b=0.9245)

[4] negativecomm(X1, X2) :- negativebehavior(X1, X2) and pprotests(X1, X2).
  Training Acc: 0.8855 | Eval Acc: 0.9026
  Val Coverage: (N_b=107, N_r=107, N_r/N_b=1.0000)
  Train Coverage: (N_b=1853, N_r=1853, N_r/N_b=1.0000)

[5] negativecomm(X1, X2) :- negativebehavior(X1, X2) and negativebehavior(X2, X1) and blockpositionindex(X1, X2).
  Training Acc: 0.8695 | Eval Acc: 0.8817
  Val Coverage: (N_b=111, N_r=100, N_r/N_b=0.9009)
  Train Coverage: (N_b=2130, N_r=1860, N_r/N_b=0.8732)

[6] negativecomm(X1, X2) :- negativebehavior(X1, X2) and commonbloc0(X1, X2).
  Training Acc: 0.8627 | Eval Acc: 0.8805
  Val Coverage: (N_b=98, N_r=93, N_r/N_b=0.9490)
  Train Coverage: (N_b=1862, N_r=1671, N_r/N_b=0.8974)

[7] negativecomm(X1, X2) :- negativecomm(X2, X1) and violentactions(X2, X1).
  Training Acc: 0.8754 | Eval Acc: 0.8677
  Val Coverage: (N_b=127, N_r=102, N_r/N_b=0.8031)
  Train Coverage: (N_b=2813, N_r=2250, N_r/N_b=0.7999)

[8] negativecomm(X1, X2) :- negativecomm(X2, X1).
  Training Acc: 0.8754 | Eval Acc: 0.8677
  Val Coverage: (N_b=127, N_r=102, N_r/N_b=0.8031)
  Train Coverage: (N_b=2813, N_r=2250, N_r/N_b=0.7999)

[9] negativecomm(X1, X2) :- negativebehavior(X2, X1) and blockpositionindex(X1, X2) and negativecomm(X2, X1).
  Training Acc: 0.8595 | Eval Acc: 0.8631
  Val Coverage: (N_b=113, N_r=93, N_r/N_b=0.8230)
  Train Coverage: (N_b=2111, N_r=1769, N_r/N_b=0.8380)

[10] negativecomm(X1, X2) :- accusation(X1, X2) and accusation(X2, X1).
  Training Acc: 0.8757 | Eval Acc: 0.8619
  Val Coverage: (N_b=96, N_r=84, N_r/N_b=0.8750)
  Train Coverage: (N_b=2060, N_r=1876, N_r/N_b=0.9107)
\end{lstlisting}
\vspace{5mm}

\subsection{UMLS Dataset}
\addcontentsline{toc}{subsection}{UMLS Dataset}

The \textit{UMLS} dataset consists of biomedical entities connected through a large set of
heterogeneous semantic relations.
The task requires learning abstract ontological relationships from highly noisy and densely
connected knowledge graphs.

\subsubsection*{\texttt{isa}$(X_1, X_2)$}
\addcontentsline{toc}{subsubsection}{\texttt{isa}$(X_1, X_2)$}

The target predicate \texttt{isa}$(X_1, X_2)$ captures hierarchical semantic relationships
between biomedical concepts.
Below, we present representative rules extracted by ANDRE for this predicate.
Rules are shown in their raw, prompt-style form as produced by the model.

\begin{lstlisting}[
    style=stateSnapshot,
    caption={Extracted Rules for \texttt{isa}$(X_1, X_2)$ (UMLS Dataset)},
    label={lst:umls_isa_rules}
]
[1] isa(X1, X2) :- isa(X3, X2) and interacts_with(X1, X3).
  Training Acc: 0.9449 | Eval Acc: 0.9466

[2] isa(X1, X2) :- isa(X3, X2) and conceptually_related_to(X3, X1).
  Training Acc: 0.9437 | Eval Acc: 0.9453

[3] isa(X1, X2) :- connected_to(X3, X1) and practices(X3, X2).
  Training Acc: 0.9437 | Eval Acc: 0.9453

[4] isa(X1, X2) :- not(affects(X3, X1)) and conceptual_part_of(X3, X2).
  Training Acc: 0.9430 | Eval Acc: 0.9449

[5] isa(X1, X2) :- not(isa(X3, X2)) and conceptual_part_of(X1, X3).
  Training Acc: 0.9429 | Eval Acc: 0.9445
\end{lstlisting}

\vspace{3mm}
\subsubsection*{\texttt{interacts\_with}$(X_1, X_2)$}
\addcontentsline{toc}{subsubsection}{\texttt{interacts\_with}$(X_1, X_2)$}

The target predicate \texttt{interacts\_with}$(X_1, X_2)$ captures functional, biochemical,
or causal interactions between biomedical entities in the UMLS knowledge base.
The rules below illustrate how ANDRE infers interaction patterns by leveraging ontological
relations (e.g., \texttt{isa}), negated constraints, and higher-order relational dependencies.

\begin{lstlisting}[
    style=stateSnapshot,
    caption={Extracted Rules for \texttt{interacts\_with}$(X_1, X_2)$ (UMLS Dataset)},
    label={lst:umls_interacts_rules}
]
[1] interacts_with(X1, X2) :- isa(X2, X1) and not(associated_with(X2, X3)) and not(interacts_with(X2, X1)) and not(ingredient_of(X2, X3)).
  Training Acc: 0.8920 | Eval Acc: 0.8946

[2] interacts_with(X1, X2) :- isa(X2, X1) and not(interacts_with(X2, X1)) and not(part_of(X3, X1)) and not(measures(X3, X2)).
  Training Acc: 0.8912 | Eval Acc: 0.8938

[3] interacts_with(X1, X2) :- not(location_of(X3, X2)) and isa(X1, X2) and not(interacts_with(X2, X1)) and not(complicates(X2, X3)).
  Training Acc: 0.8865 | Eval Acc: 0.8879

[4] interacts_with(X1, X2) :- interacts_with(X1, X3) and interacts_with(X3, X2).
  Training Acc: 0.8801 | Eval Acc: 0.8823

[5] interacts_with(X1, X2) :- associated_with(X1, X3) and performs(X2, X3).
  Training Acc: 0.8722 | Eval Acc: 0.8744

[6] interacts_with(X1, X2) :- co_occurs_with(X1, X3) and indicates(X2, X3).
  Training Acc: 0.8722 | Eval Acc: 0.8744

[7] interacts_with(X1, X2) :- treats(X3, X2) and developmental_form_of(X1, X3).
  Training Acc: 0.8722 | Eval Acc: 0.8744

[8] interacts_with(X1, X2) :- ingredient_of(X1, X3) and interconnects(X3, X2).
  Training Acc: 0.8722 | Eval Acc: 0.8744

[9] interacts_with(X1, X2) :- result_of(X2, X3) and adjacent_to(X1, X3).
  Training Acc: 0.8722 | Eval Acc: 0.8744
\end{lstlisting}

\vspace{3mm}

\section{Runtime Comparison}
\label{tab:train-time}

Table~\ref{tab:runtime_comparison} reports the total running time required by
NTP$\lambda$, NeuralLP, DFORL, and ANDRE to generate complete sets of logic programs
on the Countries, Nations, and UMLS datasets.
The results highlight substantial differences in computational efficiency across methods.
NeuralLP consistently achieves the shortest runtime due to its purely neural formulation,
while NTP$\lambda$ incurs significantly higher computational cost, particularly on larger
and more relationally complex datasets such as Nations and UMLS.
ANDRE exhibits intermediate runtime performance, remaining substantially faster than
NTP$\lambda$ while scaling more favorably than DFORL on dense relational datasets.
Overall, these results indicate that ANDRE provides a practical trade-off between
computational efficiency and symbolic rule extraction capability.

\begin{table}[h!]
\centering
\caption{Running time (in minutes) of NTP$\lambda$, NeuralLP, DFORL, and ANDRE when generating all logic programs on Countries, Nations, and UMLS datasets. Bold values indicate the best performance.}
\label{tab:runtime_comparison}
\begin{tabular}{lcccc}
\toprule
\textbf{Dataset} & \textbf{NTP$\lambda$} & \textbf{NeuralLP} & \textbf{DFORL} & \textbf{ANDRE} \\
\midrule
Countries-S1 & 30.1  & \textbf{0.1} & 11    & 3.6 \\
Countries-S2 & 43.2  & \textbf{0.1} & 8     & 3.2 \\
Countries-S3 & 252.4 & \textbf{0.1} & 88.4  & 5.5 \\
Nations      & 600.4 & \textbf{0.4} & 0.6   & 1.2 \\
UMLS         & 150.4 & \textbf{0.6} & 5.2   & 3.5 \\
% Kinship      & 320.0 & \textbf{2.3} & 417   & -- \\
\bottomrule
\end{tabular}
\end{table}

\section{Further Results and Baselines}
\label{ap:further-results}

Table~\ref{tab:uwcse_alzheimers_accuracy} summarizes the predictive accuracy of
CILP++, D-LFIT, DFORL, and ANDRE on the UW-CSE and Alzheimers-amine datasets.
On UW-CSE, ANDRE achieves the highest accuracy, substantially outperforming
prior symbolic and differentiable rule learning baselines.
On Alzheimers-amine, DFORL attains the best performance, while ANDRE achieves
competitive accuracy close to the state of the art.
These results indicate that ANDRE generalizes effectively across diverse relational
domains while maintaining strong predictive performance.

\begin{table}[h]
\centering
\caption{Accuracy results (\%) on the UW-CSE and Alzheimers-amine datasets.
Baseline results are taken from~\citep{gao2024differentiable}.
Bold values indicate the best performance for each dataset.}
\label{tab:uwcse_alzheimers_accuracy}
\begin{tabular}{lcccc}
\toprule
\textbf{Dataset} & \textbf{CILP++} & \textbf{D-LFIT} & \textbf{DFORL} & \textbf{ANDRE} \\
\midrule
UW-CSE             & 81.98 & 79.44 & 46.90 & \textbf{95.19} \\
Alzheimers-amine   & 78.70 & 67.75 & \textbf{99.13} & 97.04 \\
\bottomrule
\end{tabular}
\end{table}

\subsection{Alzheimers-amine Dataset}
\addcontentsline{toc}{subsection}{Alzheimers-amine Dataset}

The \textit{Alzheimers-amine} dataset focuses on molecular interaction patterns
and chemical property relations relevant to Alzheimer’s disease drug discovery.
Predicates encode structural, functional, and physicochemical relationships
between molecular components.

\subsubsection*{\texttt{great\_ne}$(X_1, X_2)$}
\addcontentsline{toc}{subsubsection}{\texttt{great\_ne}$(X_1, X_2)$}

The target predicate \texttt{great\_ne}$(X_1, X_2)$ models a strong negative interaction
between molecular entities.
Below, we present representative rules extracted by ANDRE, illustrating both
simple logical constraints and more complex recursive and relational dependencies.

\begin{lstlisting}[
    style=stateSnapshot,
    caption={Extracted Rules for \texttt{great\_ne}$(X_1, X_2)$ (Alzheimers-amine Dataset)},
    label={lst:alzheimers_great_ne_rules}
]
[1] great_ne(X1, X2) :- not(great_ne(X2, X1)).
  Training Acc: 0.9216 | Eval Acc: 0.9217

[2] great_ne(X1, X2) :- not(great_ne(X2, X1)) and great_ne(X3, X2) and not(r_subst_1(X1, X3)).
  Training Acc: 0.8896 | Eval Acc: 0.8786

[3] great_ne(X1, X2) :- x_subst(X2, X3) and r_subst_1(X1, X3).
  Training Acc: 0.7689 | Eval Acc: 0.7609

[4] great_ne(X1, X2) :- gt(X1, X3) and great_pi_acc(X2, X3).
  Training Acc: 0.7689 | Eval Acc: 0.7609

[5] great_ne(X1, X2) :- not(ring_subst_4(X2, X3)) and ring_subst_4(X3, X1).
  Training Acc: 0.7689 | Eval Acc: 0.7609

[6] great_ne(X1, X2) :- great_ne(X3, X2) and flex(X1, X3).
  Training Acc: 0.7689 | Eval Acc: 0.7609

[7] great_ne(X1, X2) :- r_subst_2(X3, X2) and ring_subst_2(X3, X1).
  Training Acc: 0.7689 | Eval Acc: 0.7609

[8] great_ne(X1, X2) :- great_ne(X3, X2) and ring_subst_2(X3, X1).
  Training Acc: 0.7689 | Eval Acc: 0.7609

[9] great_ne(X1, X2) :- pi_doner(X3, X1) and not(ring_substitutions(X2, X3)).
  Training Acc: 0.7689 | Eval Acc: 0.7609

[10] great_ne(X1, X2) :- great_ne(X2, X3) and pi_doner(X3, X1).
  Training Acc: 0.7689 | Eval Acc: 0.7609

\end{lstlisting}
\vspace{3mm}

\subsection{UW-CSE Dataset}
\addcontentsline{toc}{subsection}{UW-CSE Dataset}

The \textit{UW-CSE} dataset models academic relationships within a university domain,
including roles, courses, projects, and advising relationships.
The task is to infer latent advisory relations from heterogeneous academic facts.

\subsubsection*{\texttt{advisedby}$(X_1, X_2)$}
\addcontentsline{toc}{subsubsection}{\texttt{advisedby}$(X_1, X_2)$}

The target predicate \texttt{advisedby}$(X_1, X_2)$ represents an academic advising
relationship between individuals.
Below, we present representative rules extracted by ANDRE, highlighting recursive,
role-based, and relational patterns commonly observed in academic environments.

\begin{lstlisting}[
    style=stateSnapshot,
    caption={Extracted Rules for \texttt{advisedby}$(X_1, X_2)$ (UW-CSE Dataset)},
    label={lst:uwcse_advisedby_rules}
]
[1] advisedby(X1, X2) :- advisedby(X3, X1) and yearsinprogram(X3, X2).

[2] advisedby(X1, X2) :- courselevel(X1, X3) and hasposition(X2, X3).

[3] advisedby(X1, X2) :- professor(X3, X2) and taughtby(X3, X1).

[4] advisedby(X1, X2) :- courselevel(X2, X3) and not(professor(X3, X1)).

[5] advisedby(X1, X2) :- advisedby(X3, X1) and not(inphase(X3, X2)).

[6] advisedby(X1, X2) :- hasposition(X2, X3) and ta(X3, X1).

[7] advisedby(X1, X2) :- not(hasposition(X3, X1)) and projectmember(X2, X3).

[8] advisedby(X1, X2) :- inphase(X1, X3) and professor(X2, X3).

[9] advisedby(X1, X2) :- not(courselevel(X3, X2)) and publication(X1, X3).

[10] advisedby(X1, X2) :- hasposition(X2, X3) and inphase(X3, X1).

\end{lstlisting}

\vspace{3mm}

\onecolumn
\section{Synthetic Datasets Tabular Results}
\label{ap:results}
\begin{table*}[!h]
\centering
% \scriptsize
\caption{Comparison of Rule Extraction Performance between ANDRE and DFORL on Complex Synthetic Datasets with Varying Number of Subrules}
\vspace{0.1in}
\renewcommand{\arraystretch}{1.2}
% \scriptsize
\begin{tabular}{cccccccc}
\toprule
\multirow{3}{*}{\centering\makecell{\textbf{Dataset}}} &
\multirow{3}{*}{\centering\makecell{\textbf{Sample}\\\textbf{Size}}} &
\multicolumn{4}{c}{\textbf{Accuracy}} &
\multicolumn{2}{c}{\multirow{2}{*}{\centering\makecell{\textbf{Rule Extraction}\\\textbf{Success}}}} \\
\cline{3-6}
& &
\multicolumn{2}{c}{\textbf{ANDRE}} &
\multicolumn{2}{c}{\textbf{DFORL}} &
 & \\
\cline{3-4} \cline{5-6} \cline{7-8}
& & \textbf{Train} & \textbf{Test} & \textbf{Train} & \textbf{Test} & \textbf{ANDRE} & \textbf{DFORL} \\
\toprule
\multirow{4}{*}{$\texttt{R}_1$} & $20$ & $0.95$ & $0.80$ & $0.85$ & $0.84$ & \xmark & \xmark \\
 & $50$ & $0.98$ & $0.96$ & $0.82$ & $0.84$ & \cmark & \xmark \\
 & $100$ & $0.95$ & $0.96$ & $0.83$ & $0.85$ & \cmark & \cmark \\
 & $200$ & $0.96$ & $0.96$ & $0.86$ & $0.84$ & \cmark & \cmark \\
\midrule
\multirow{5}{*}{$\texttt{R}_2$} & $50$ & $0.94$ & $0.88$ & $0.78$ & $0.74$ & \xmark & \xmark \\
 & $100$ & $0.95$ & $0.96$ & $0.74$ & $0.73$ & \cmark & \xmark \\
 & $200$ & $0.95$ & $0.95$ & $0.73$ & $0.73$ & \cmark & \xmark \\
 & $500$ & $0.95$ & $0.95$ & $0.75$ & $0.73$ & \cmark & \xmark \\
 & $1000$ & $0.96$ & $0.97$ & $0.74$ & $0.72$ & \cmark & \xmark \\
\midrule
\multirow{5}{*}{$\texttt{R}_3$} & $200$ & $0.87$ & $0.83$ & $0.75$ & $0.79$ & \xmark & \xmark \\
 & $500$ & $0.95$ & $0.95$ & $0.68$ & $0.68$ & \cmark & \xmark \\
 & $1000$ & $0.90$ & $0.90$ & $0.68$ & $0.70$ & \cmark & \xmark \\
 & $1500$ & $0.90$ & $0.90$ & $0.68$ & $0.67$ & \cmark & \xmark \\
 & $2000$ & $0.90$ & $0.89$ & $0.68$ & $0.65$ & \cmark & \xmark \\
\toprule
\end{tabular}
\label{tab:andre_vs_dforl_full}
\end{table*}

\begin{table*}[!h]
\centering
\caption{Comparison of Rule Extraction Performance between ANDRE and DFORL on Noisy Synthetic Datasets}
\renewcommand{\arraystretch}{1.2}
% \scriptsize
\begin{tabular}{ccccccccc}
\toprule
\multirow{3}{*}{\centering\makecell{\textbf{Dataset}}} &
\multirow{3}{*}{\centering\makecell{\textbf{Sample}\\\textbf{Size}}} &
\multirow{3}{*}{\centering\makecell{\textbf{Noise}\\\textbf{(\%)}}} &
\multicolumn{4}{c}{\textbf{Accuracy}} &
\multicolumn{2}{c}{\multirow{2}{*}{\centering\makecell{\textbf{Rule Extraction}\\\textbf{Success}}}} \\
\cline{4-7}
& & & \multicolumn{2}{c}{\textbf{ANDRE}} & \multicolumn{2}{c}{\textbf{DFORL}} & & \\
\cline{4-5} \cline{6-7} \cline{8-9}
& & & \textbf{Train} & \textbf{Test} & \textbf{Train} & \textbf{Test} & \textbf{ANDRE} & \textbf{DFORL} \\
\midrule
\multirow{4}{*}{$\texttt{R}_4$} & \multirow{4}{*}{$200$} & $10$ & $0.86$ & $0.86$ & $0.71$ & $0.66$ & \cmark & \cmark \\
 &  & $20$ & $0.79$ & $0.74$ & $0.68$ & $0.61$ & \cmark & \xmark \\
 &  & $25$ & $0.74$ & $0.71$ & $0.65$ & $0.60$ & \cmark & \xmark \\
 &  & $30$ & $0.69$ & $0.61$ & $0.64$ & $0.58$ & \xmark & \xmark \\
\midrule
\multirow{5}{*}{$\texttt{R}_5$} & \multirow{5}{*}{$500$} & $5$ & $0.90$ & $0.90$ & $0.62$ & $0.56$ & \cmark & \xmark \\
 &  & $15$ & $0.82$ & $0.82$ & $0.58$ & $0.54$ & \cmark & \xmark \\
 &  & $25$ & $0.73$ & $0.72$ & $0.56$ & $0.54$ & \cmark & \xmark \\
 &  & $35$ & $0.62$ & $0.55$ & $0.54$ & $0.51$ & \xmark & \xmark \\
 &  & $45$ & $0.62$ & $0.55$ & $0.49$ & $0.47$ & \xmark & \xmark \\
\midrule
\multirow{5}{*}{$\texttt{R}_6$} & \multirow{5}{*}{$1000$} & $5$ & $0.84$ & $0.82$ & $0.49$ & $0.49$ & \cmark & \xmark \\
 &  & $10$ & $0.80$ & $0.78$ & $0.49$ & $0.48$ & \cmark & \xmark \\
 &  & $15$ & $0.76$ & $0.76$ & $0.49$ & $0.48$ & \cmark & \xmark \\
 &  & $20$ & $0.73$ & $0.71$ & $0.48$ & $0.48$ & \xmark & \xmark \\
 &  & $25$ & $0.70$ & $0.67$ & $0.49$ & $0.48$ & \xmark & \xmark \\
\toprule
\end{tabular}
\label{tab:andre_vs_dforl_noisy}
\end{table*}

% \begin{table}[h!]
% \centering
% \caption{Ablation Study Results}
% \label{tab:ablation}
% \renewcommand{\arraystretch}{1.2}
% \begin{tabular}{lcc}
% \hline
% \textbf{Method} & ANDRE w/o Attention & ANDRE \\
% \hline
% \textbf{Accuracy} & $0.76$ & $0.92$ \\
% \hline
% \end{tabular}
% \end{table}

\end{document}

%% file: math_commands.tex
\usepackage{amsmath,amsfonts,bm}

\def\eqref#1{equation~\ref{#1}}
\def\1{\bm{1}}

\DeclareMathAlphabet{\mathsfit}{\encodingdefault}{\sfdefault}{m}{sl}
\SetMathAlphabet{\mathsfit}{bold}{\encodingdefault}{\sfdefault}{bx}{n}

%% file: iclr2025_conference.bib
@article{muggleton2012ilp,
  title={ILP turns 20: biography and future challenges},
  author={Muggleton, Stephen and De Raedt, Luc and Poole, David and Bratko, Ivan and Flach, Peter and Inoue, Katsumi and Srinivasan, Ashwin},
  journal={Machine learning},
  volume={86},
  pages={3--23},
  year={2012},
  publisher={Springer}
}

@article{cropper2022inductive,
  title={Inductive logic programming at 30: a new introduction},
  author={Cropper, Andrew and Duman{\v{c}}i{\'c}, Sebastijan},
  journal={Journal of Artificial Intelligence Research},
  volume={74},
  pages={765--850},
  year={2022}
}

@inproceedings{cropper2015logical,
  title={Logical minimisation of meta-rules within meta-interpretive learning},
  author={Cropper, Andrew and Muggleton, Stephen H},
  booktitle={Inductive Logic Programming: 24th International Conference, ILP 2014, Nancy, France, September 14-16, 2014, Revised Selected Papers},
  pages={62--75},
  year={2015},
  organization={Springer}
}

@article{cropper2021learning,
  title={Learning programs by learning from failures},
  author={Cropper, Andrew and Morel, Rolf},
  journal={Machine Learning},
  volume={110},
  number={4},
  pages={801--856},
  year={2021},
  publisher={Springer}
}

@article{evans2018learning,
  title={Learning explanatory rules from noisy data},
  author={Evans, Richard and Grefenstette, Edward},
  journal={Journal of Artificial Intelligence Research},
  volume={61},
  pages={1--64},
  year={2018}
}

@inproceedings{
dong2018neural,
title={{Neural Logic Machines}},
author={Honghua Dong and Jiayuan Mao and Tian Lin and Chong Wang and Lihong Li and Denny Zhou},
booktitle={International Conference on Learning Representations},
year={2019},
url={https://openreview.net/forum?id=B1xY-hRctX},
}

@inproceedings{gao2022learning,
  title     = {{Learning First-Order Rules with Differentiable Logic Program Semantics}},
  author    = {Gao, Kun and Inoue, Katsumi and Cao, Yongzhi and Wang, Hanpin},
  booktitle = {Proceedings of the Thirty-First International Joint Conference on
               Artificial Intelligence, {IJCAI-22}},
  publisher = {International Joint Conferences on Artificial Intelligence Organization},
  editor    = {Lud De Raedt},
  pages     = {3008--3014},
  year      = {2022},
  month     = {7},
  note      = {Main Track},
  doi       = {10.24963/ijcai.2022/417},
  url       = {https://doi.org/10.24963/ijcai.2022/417},
}

@inproceedings{sen2022neuro,
  title={Neuro-symbolic inductive logic programming with logical neural networks},
  author={Sen, Prithviraj and de Carvalho, Breno WSR and Riegel, Ryan and Gray, Alexander},
  booktitle={Proceedings of the AAAI conference on artificial intelligence},
  volume={36},
  number={8},
  pages={8212--8219},
  year={2022}
}

@article{serafini2016logic,
  title={{Logic Tensor Networks}},
  author={Samy Badreddine and Artur S. d'Avila Garcez and Luciano Serafini and Mike Spranger},
  journal={Artificial Intelligence},
  year={2020},
  volume={303},
  pages={103649},
  url={https://api.semanticscholar.org/CorpusID:229678739}
}

@inproceedings{shindo2021differentiable,
  title={Differentiable inductive logic programming for structured examples},
  author={Shindo, Hikaru and Nishino, Masaaki and Yamamoto, Akihiro},
  booktitle={Proceedings of the AAAI Conference on Artificial Intelligence},
  volume={35},
  number={6},
  pages={5034--5041},
  year={2021}
}

@article{shindo2023alpha,
  title={
 $\alpha$-ILP: thinking visual scenes as differentiable logic programs},
  author={Hikaru Shindo and Viktor Pfanschilling and Devendra Singh Dhami and Kristian Kersting},
  journal={Machine Learning},
  year={2023},
  volume={112},
  pages={1465-1497},
  url={https://api.semanticscholar.org/CorpusID:257555026}
}

@article{gao2024differentiable,
  title={A differentiable first-order rule learner for inductive logic programming},
  author={Gao, Kun and Inoue, Katsumi and Cao, Yongzhi and Wang, Hanpin},
  journal={Artificial Intelligence},
  volume={331},
  pages={104108},
  year={2024},
  publisher={Elsevier}
}

@article{garcez2023neurosymbolic,
  title={{Neurosymbolic AI}: The 3rd wave},
  author={Garcez, Artur d’Avila and Lamb, Luis C},
  journal={Artificial Intelligence Review},
  volume={56},
  number={11},
  pages={12387--12406},
  year={2023},
  publisher={Springer}
}

@book{shakarian2023neuro,
  title={Neuro Symbolic Reasoning and Learning},
  author={Shakarian, Paulo and Baral, Chitta and Simari, Gerardo I and Xi, Bowen and Pokala, Lahari},
  year={2023},
  publisher={Springer}
}

@article{muggleton2015meta,
  title={Meta-interpretive learning of higher-order dyadic datalog: Predicate invention revisited},
  author={Muggleton, Stephen H and Lin, Dianhuan and Tamaddoni-Nezhad, Alireza},
  journal={Machine Learning},
  volume={100},
  number={1},
  pages={49--73},
  year={2015},
  publisher={Springer}
}

@article{hu2025minimalist,
  title={{Minimalist Softmax Attention Provably Learns Constrained Boolean Functions}},
  author={Hu, Jerry Yao-Chieh and Zhang, Xiwen and Su, Maojiang and Song, Zhao and Liu, Han},
  journal={arXiv preprint arXiv:2505.19531},
  year={2025}
}

@article{payani2019inductive,
  title={{Inductive Logic Programming via Differentiable Deep Neural Logic Networks}},
  author={Ali Payani and Faramarz Fekri},
  journal={ArXiv},
  year={2019},
  volume={abs/1906.03523},
  url={https://api.semanticscholar.org/CorpusID:182952646}
}

@Article{sharifi2023symbolic,
AUTHOR = {Sharifi, Iman and Yildirim, Mustafa and Fallah, Saber},
TITLE = {{Symbolic Imitation Learning: From Black-Box to Explainable Driving Policies}},
JOURNAL = {Applied Sciences},
VOLUME = {15},
YEAR = {2025},
NUMBER = {23},
ARTICLE-NUMBER = {12464},
URL = {https://www.mdpi.com/2076-3417/15/23/12464},
ISSN = {2076-3417},
DOI = {10.3390/app152312464}
}

@incollection{de2008probabilistic,
  title={Probabilistic inductive logic programming},
  author={De Raedt, Luc and Kersting, Kristian},
  booktitle={Probabilistic inductive logic programming: theory and applications},
  pages={1--27},
  year={2008},
  publisher={Springer}
}

@article{richardson2006markov,
  title={Markov logic networks},
  author={Richardson, Matthew and Domingos, Pedro},
  journal={Machine learning},
  volume={62},
  pages={107--136},
  year={2006},
  publisher={Springer}
}

@article{yang2017differentiable,
  title={Differentiable learning of logical rules for knowledge base reasoning},
  author={Yang, Fan and Yang, Zhilin and Cohen, William W},
  journal={Advances in neural information processing systems},
  volume={30},
  year={2017}
}

@article{socher2013reasoning,
  title={Reasoning with neural tensor networks for knowledge base completion},
  author={Socher, Richard and Chen, Danqi and Manning, Christopher D and Ng, Andrew},
  journal={Advances in neural information processing systems},
  volume={26},
  year={2013}
}

@book{hajek2001metamathematics,
  title={Metamathematics of fuzzy logic},
  author={H{\'a}jek, Petr},
  volume={4},
  year={2001},
  publisher={Springer Science \& Business Media}
}

@inproceedings{inoue2009discovering,
  title={Discovering rules by meta-level abduction},
  author={Inoue, Katsumi and Furukawa, Koichi and Kobayashi, Ikuo and Nabeshima, Hidetomo},
  booktitle={International Conference on Inductive Logic Programming},
  pages={49--64},
  year={2009},
  organization={Springer}
}

@article{Colelough2025NeuroSymbolicAI,
  title={{Neuro-Symbolic AI in 2024: A Systematic Review}},
  author={Brandon Curtis Colelough and William Regli},
  journal={ArXiv},
  year={2025},
  volume={abs/2501.05435},
  url={https://api.semanticscholar.org/CorpusID:274180938}
}

@inproceedings{bouchard2015approximate,
  title={{On Approximate Reasoning Capabilities of Low-Rank Vector Spaces.}},
  author={Bouchard, Guillaume and Singh, Sameer and Trouillon, Theo},
  booktitle={AAAI spring symposia},
  year={2015}
}

@inproceedings{Kok2007StatisticalPI,
  title={Statistical predicate invention},
  author={Stanley Kok and Pedro M. Domingos},
  booktitle={International Conference on Machine Learning},
  year={2007},
  url={https://api.semanticscholar.org/CorpusID:6911541}
}

@InProceedings{uwcse2005,
author="Davis, Jesse
and Burnside, Elizabeth
and de Castro Dutra, In{\^e}s
and Page, David
and Costa, V{\'i}tor Santos",
editor="Gama, Jo{\~a}o
and Camacho, Rui
and Brazdil, Pavel B.
and Jorge, Al{\'i}pio M{\'a}rio
and Torgo, Lu{\'i}s",
title="An Integrated Approach to Learning Bayesian Networks of Rules",
booktitle="Machine Learning: ECML 2005",
year="2005",
publisher="Springer Berlin Heidelberg",
address="Berlin, Heidelberg",
pages="84--95",
isbn="978-3-540-31692-3"
}

@inproceedings{wang-pan-2022-deep,
    title = "Deep Inductive Logic Reasoning for Multi-Hop Reading Comprehension",
    author = "Wang, Wenya  and
      Pan, Sinno",
    editor = "Muresan, Smaranda  and
      Nakov, Preslav  and
      Villavicencio, Aline",
    booktitle = "Proceedings of the 60th Annual Meeting of the Association for Computational Linguistics (Volume 1: Long Papers)",
    month = may,
    year = "2022",
    address = "Dublin, Ireland",
    publisher = "Association for Computational Linguistics",
    url = "https://aclanthology.org/2022.acl-long.343/",
    doi = "10.18653/v1/2022.acl-long.343",
    pages = "4999--5009"
}

@article{vaswani2017attention,
  title={Attention is all you need},
  author={Vaswani, Ashish and Shazeer, Noam and Parmar, Niki and Uszkoreit, Jakob and Jones, Llion and Gomez, Aidan N and Kaiser, {\L}ukasz and Polosukhin, Illia},
  journal={Advances in neural information processing systems},
  volume={30},
  year={2017}
}

@inproceedings{acharya2025integrating,
  title={Integrating neurosymbolic AI in advanced air mobility: a comprehensive survey},
  author={Acharya, Kamal and Sharifi, Iman and Lad, Mehul and Sun, Liang and Song, Houbing},
  booktitle={Proceedings of the Thirty-Fourth International Joint Conference on Artificial Intelligence},
  pages={10362--10370},
  year={2025}
}
